\documentclass[journal]{IEEEtran}
\usepackage[utf8]{inputenc}
\usepackage{amsmath,amsthm,amsfonts}
\usepackage{xcolor}
\usepackage{svg}
\usepackage{caption}
\usepackage{subcaption}

\newcommand{\SC}{{\mathcal{C}}}
\newcommand{\SD}{{\mathcal{D}}}
\newcommand{\SX}{{\mathcal{X}}}
\newcommand{\RR}{\mathbb{R}}
\newcommand{\op}{\operatorname}
\newtheorem{proposition}{Proposition}[section]
\newtheorem{theorem}[proposition]{Theorem}
\newtheorem{corollary}[proposition]{Corollary}
\newtheorem{remark}[proposition]{Remark}

\title{Metric Tools for Sensitivity Analysis with Applications to Neural Networks}

\usepackage{hyperref}
\definecolor{darkred}{rgb}{1,0,0} 
\definecolor{darkgreen}{rgb}{0,1,0}
\definecolor{darkblue}{rgb}{0,0,1}
\hypersetup{colorlinks,
            linkcolor=darkblue,
            filecolor=darkblue,
            urlcolor=darkred,
            citecolor=darkgreen}

\usepackage{tikz}

\definecolor{lime}{HTML}{A6CE39}
\DeclareRobustCommand{\orcidicon}{%
    \begin{tikzpicture}
    \draw[lime, fill=lime] (0,0) 
    circle [radius=0.16] 
    node[white] {{\fontfamily{qag}\selectfont \tiny ID}};
    \draw[white, fill=white] (-0.0625,0.095) 
    circle [radius=0.007];
    \end{tikzpicture}
    \hspace{-2mm}
}
\newcommand{\orcid}[1]{\href{https://orcid.org/#1}{\orcidicon}}
\author{Jaime Pizarroso \orcid{0000-0002-1806-9191}, David Alfaya \orcid{0000-0002-4247-1498}, José Portela \orcid{0000-0002-7839-8982} and Antonio Muñoz \orcid{0000-0002-8844-441X} %
\thanks{The second author was supported by MICINN grant PID2019-108936GB-C21.}}

\date{July 2022}

\begin{document}
\maketitle
\begin{abstract}
As Machine Learning models are considered for autonomous decisions with significant social impact, the need for understanding how these models work rises rapidly. Explainable Artificial Intelligence (XAI) aims to provide interpretations for predictions made by Machine Learning models, in order to make the model trustworthy and more transparent for the user. For example, selecting relevant input variables for the problem directly impacts the model's ability to learn and make accurate predictions, so obtaining information about input importance play a crucial role when training the model. One of the main XAI techniques to obtain input variable importance is the sensitivity analysis based on partial derivatives. However, existing literature of this method provide no justification of the aggregation metrics used to retrieved information from the partial derivatives. 

In this paper, a theoretical framework is proposed to study sensitivities of ML models using metric techniques. From this metric interpretation, a complete family of new quantitative metrics called $\alpha$-curves is extracted. These $\alpha$-curves provide information with greater depth on the importance of the input variables for a machine learning model than existing XAI methods in the literature.  We demonstrate the effectiveness of the $\alpha$-curves using synthetic and real datasets, comparing the results against other XAI methods for variable importance and validating the analysis results with the ground truth or literature information. 
\end{abstract}


\section{Introduction}
\label{section:introduction}

Artificial Intelligence (AI) has gained popularity in the last few years, exceeding expectations in a variety of fields. One example is in the field of predictive analytics, where Machine Learning (ML) models are used to make predictions about future events based on data patterns (\cite{Kohli_2018, Shamout_2021, Rashid_2021}). As data availability increases and more complex problems are tackled, models with a higher number of parameters are needed to accurately learn from data \cite{dean_2022}.

This increase in the complexity of the model is associated with a lack of interpretability and affects its credibility and trust. Explainable Artificial Intelligence (XAI) is a relatively recent field whose main objective is to make ML models trustworthy \cite{Samek_2021, Samek2019-mm}. There is a growing interest in XAI, as it can help address some of the concerns around responsible AI. For example, if an AI system is used to make decisions that could have significant social impact (such as in healthcare or finance), then it is important that there is a way to understand how and why the system arrived at its decisions \cite{Barredo_2020, Benjamins_2019, Preece_2018}. This would allow for accountability and transparency, two key components of responsible AI. Furthermore, XAI techniques are not only useful for validating a ML model, but can also be used to retrieve information from the dataset itself. This information can be used to corroborate the prior knowledge of a field, which is an important aspect for evaluating the quality of a model \cite{Carvalho_2019}.

 One important aspect of explainable models is the ability to understand and interpret the factors that drive their predictions. Variable importance metrics are a key tool in this endeavor, as they provide insight into which features of the input data are most important in determining the model's output.
 This can be useful for building trust in a model, as well as for identifying potential biases or errors in the model. On the one hand, some ML model topologies such as decision trees or linear regression are inherently transparent and provide variable importance measures based on model parameters. On the other hand, neural networks (NN) models are hard to interpret and additional methods must be used to calculate variable importances.

NN models have gained popularity in recent years due to their ability to learn complex patterns from data and make accurate predictions. Despite their impressive performance, NN models are not commonly used in critical applications due to their lack of interpretability \cite{Cheng_2018}. Consequently, improving the methods to provide interpretability to NN models would unlock the potential of this type of models in applications where their adaptation capabilities provide a higher added value than traditional interpretable models \cite{Samek_2021}.

Some techniques are already available to estimate variable importance of NN models. The most commonly used are input permutation \cite{input_permutation_2021}, which consists in perturbing the input data and observing the effect on the model's output and SHAP (SHapley Additive exPlanations) \cite{Shapley_1953}, which assigns an importance value to each feature by averaging over all possible coalition of features. These techniques have notable advantages: they do not depend of the topology of the ML model being analyzed (model-agnostic method) and provide quantifiable information of variable importance \cite{perm_shap_2020, Hariharan2022}. 

However, all of these techniques provide a global variable importance without analyzing the distribution of the local importance along the input space. This lack of information can make it difficult to understand how a feature is impacting the model's predictions in different regions of the input space, which can be critical for many applications. A potential consequence of not addressing this question is considering a variable which only affects the output in a specific region of the input space as less important overall compared to a variable with a wider-range impact, even if its individual effect is smaller. This could lead to flawed decision making and incorrect conclusions about the data, neglecting a variable with a notable effect on the output of the model.   This highlights the need for further research on developing variable importance metrics that can provide a more comprehensive understanding of feature importance.
As far as the authors are aware, there is currently no XAI method that quantitatively addresses this question.

In this paper, a method to obtain information about local and global importances of input variables is presented. The developed method combines the sensitivity analysis of ML models using partial derivatives \cite{dimopoulos_use_1995, Gevrey_2003, Gevrey_2006} with a mathematical study of certain metric spaces and operator metrics which can be associated to the model and the dataset. These partial derivatives, also called sensitivities, measure the degree to which the output of the model is affected by changes in the input variables. These sensitivities can be calculated for the samples in the input dataset, obtaining a distribution of sensitivity values for each of the input variables. Information of the model can be retrieved from these distributions, such as which variable has the greatest impact on the predictions.

In this work, a global theoretical metric interpretation of the sensitivity of a ML model which takes into account the whole dataset at once is presented and metrics which aggregate the local sensitivities across the dataset are derived from it as pointwise Lipschitz constants of a certain variation operator associated to the model and the dataset. The pointwise Lipschitz constant of a function between two metric spaces is a mathematical tool used precisely to quantify to what extent a variation of an input affects its output (see, for instance, \cite{Durand2010}). Computing local and global Lipschitz constants of the NN model itself has been used in the literature for other purposes as measures of NN robustness \cite{Szegedy2014} and deep NN stability \cite{Thomas2022}, however , to the authors' knowledge, analysing sensitivity from the perspective of this mathematical tool is a new approach. Moreover the novel proposed metric framework allows us to obtain novel sensitivity measures for general ML models with relevant applications to obtaining variable importance metrics.

Hence, a complete family of new quantitative metrics called $\alpha$-curves is extracted. These $\alpha$-curves provide information with greater depth of the variation of the output of a ML model with respect to a specific variable throughout the entire input space. This information not only allows to determine which are the most important input variables for the model, but also answers the question if there are regions in the input space where a variable is specially relevant.

 In this paper we demonstrate how the $\alpha$-curves method is able to provide valuable information about the inner workings of the model compared to other commonly used XAI methods.
Specifically, the $\alpha-$curves methodology overcomes other methods' disadvantages, as it does not assume linear relationships between output and input variable (SHAP), it provides global to local level explanations, and the computational complexity is comparable to the sensitivity analysis based on partial derivatives, substantially lower than SHAP.

The rest of the paper is organized as follows. Section  \ref{section:sota} collects a state-of-the-art review of XAI techniques applied to obtaining variable importance metrics. In section \ref{section:metricsensitivity}, the theoretical framework and the $\alpha$-curves quantitative metrics are presented. Section \ref{section:methodology-of-alpha-curves} provides a methodology to use $\alpha$-curves as sensitivity analysis method of ML models. Section \ref{section:experiments} presents examples of this methodology applied to both synthetic and real datasets, demonstrating that the method presented in this paper is able to retrieve information from ML models with greater detail than other XAI techniques. Section \ref{section:conclusions} concludes the article and presents future research lines. 

\section{state of the art}
\label{section:sota}

At the time of writing, XAI techniques can be classified following three characteristics \cite{molnar_2022}:
\begin{itemize}
    \item Intrinsic vs post-hoc: model simple structure provides intrinsic explanation for its decisions or a method shall be applied to the model after training to retrieve information.
    \item Model-specific vs model-agnostic: XAI technique is applicable to specific ML models or to any ML model.
    \item Global vs local: method explains the behaviour of the model for the entire input space or for individual points.
\end{itemize}

It is out of the scope of this paper to provide an in-depth state-of-the-art review of XAI techniques, so we only gather a review of post-hoc  regression techniques focused  on variable importance  metrics. For a broader review of XAI methods, we refer the reader to \cite{Barredo_2020, Letzgus_2022, Minh_2022, Speith_2022, bykov2020i}.  Apart from the intrinsic explainable models, such as linear regression and the family of decision trees \cite{kastner_2021}, the main variable importance XAI methods are:

\begin{enumerate}

       \item Input permutation \cite{input_permutation_2021, Gevrey_2003}. The technique involves shuffling the values of one input feature and observing the effect on the model's prediction.  The resulting change in a chosen error metric for each input permutation represents the relative importance of each input variable.
       \item Input perturbation \cite{scardi_developing_1999, Gevrey_2003}. Similar to input permutation, it consists in adding a small perturbation to each input variable while maintaining the other inputs at a constant value. The resulting change in a chosen error metric for each input perturbation represents the relative importance of each input variable.       
       \item Partial derivatives method for sensitivity analysis \cite{Gevrey_2003, Gevrey_2006, dimopoulos_use_1995, dimopoulos_neural_1999, munoz_1998, white_statistical_2001}. It performs a sensitivity analysis by computing the partial derivatives of the model output with regard to the input neurons evaluated on the samples of the training dataset (or an analogous dataset). 
       \item Shapley values \cite{Shapley_1953, strumbelj_kononenko_2010}: originated in game theory, the shapley value is essentially the average expected marginal contribution of one variable after all possible input variable combinations have been considered. An evolution of this method is the SHapley Additive exPlanation (SHAP) \cite{Lundberg_2017} method, where the Shapley Values for an ML model are calculated based on LIME (instead of calculating all combinations), reducing the computational resources.
\end{enumerate}

The following section presents a brief explanation of the sensitivity analysis based on partial derivatives, which is the basis for developing the $\alpha-$curves methodology.

\subsection{Sensitivity analysis based on partial derivatives}
 Given a  Neutral Network model trained, Feature importance is taken as the mean squared sensitivity of the output with regard to the input variable:
\begin{equation*}
\label{eqn:M_SQ_AV_S}
	S_{i}^{sq} = \sqrt{\frac{\sum_{j = 1}^{N}
            \left(s_{i} \big\rvert_{\mathbf{x}_j}\right)^2}{N}}\, ,
\end{equation*}
where $i$ is the index of the feature whose importance we want to calculate, $N$ is the number of samples in the dataset we are using for the sensitivity analysis and $s_{i} \big\rvert_{\mathbf{x}_j}$ is the sensitivity of the output $k$ with respect to the input $i$ evaluated in the sample $\mathbf{x}_j$:
\begin{equation*}
\label{eqn:sensitivity}
	s_{ik} \big\rvert_{\mathbf{x}_n} = \frac{\partial y_k}{\partial
        x_i} \left(\mathbf{x}_n\right) \, .
\end{equation*}

Another two sensitivity-based measures are presented in \cite{pizarroso_2022}: mean and standard deviation of sensitivities:
\begin{equation*}
\label{eqn:M_AV_S}
    S_{i}^{avg} = \frac{\sum_{j = 1}^{N} s_{i} \big\rvert_{\mathbf{x}_j}}{N}
\end{equation*}
\begin{equation*}
\label{eqn:STD_S}
      S_{i}^{sd} =
          \sigma\left(s_{i} \big\rvert_{\mathbf{x}_j}\right); \, \, \,  j \in
          \{1,\dots,N\} \, .
\end{equation*}
Based on these measures, the following information can be obtained from a ML model:
\begin{itemize}
    \item an input variable has a non-linear relationship with the output if $S_{i}^{sd} >> 0$.
    \item an input variable has a linear relationship with the output if $S_{i}^{sd} \approx 0$ and  $S_{i}^{avg} \neq 0$.
    \item an input variable has no relationship with the output if standard deviation $S_{i}^{sd} \approx 0$ and $S_{i}^{avg} \approx 0$.
\end{itemize}

These are useful measures to retrieve information from a ML model. In \cite{pizarroso_2022} a comparison of sensitivity analysis using partial derivatives with most of the other methods is performed. The main advantage of this method is that it provides feature importance measures together with information about the relationship between the output and the input, requiring less computational resources compared to other techniques.
 
 However, the feature importance measures give few information about the sensitivity distribution along the input space. An input variable with low sensitivities in most of the input space but with high sensitivity in certain samples would be assigned a low importance, misleading the user.  
The $\alpha-$curves method is an evolution of sensitivity analysis, providing a metric interpretation of the partial derivatives distribution. This provides extra information by not only giving the same feature importance information, but also provides information about how the sensitivity with respect to a variable is distributed in the input space.

\section{A metric interpretation of sensitivity}
\label{section:metricsensitivity}
%

Let $f:\RR^n\longrightarrow \RR^m$ be a differentiable function. Let $\SX=\{\bar{x}_i\}_{i=1}^N$ be a dataset with $\bar{x}_i=(x_{i,1},\ldots,x_{i,n})\in \mathbb{R}^n$ for each $i=1,\ldots,N$. We propose the following metric framework for analysing the sensitivity of the function $f(x_1,\ldots,x_n)$ with respect to variable $X_j$ over the dataset $\SX$. We will analyze variations of the values of $f$ at the points of $\SX$ when we perturb each point $\bar{x}_i$ with a variation in the $j$-th component of the point in the following way. We will measure the total variation of the values
$$f(x_{i,1},\ldots,x_{i,j}+h_{i},\ldots, x_{i,n})$$
when we introduce small perturbations $h_1,\ldots,h_N\in \mathbb{R}$ on the variable $X_j$ of each point $\bar{x}_1,\ldots,\bar{x}_N\in \SX$ respectively.

\sloppy In order to make this precise, we need to fix a way to measure the total variation of $f$ across the dataset $\SX$ and a way to measure the perturbation $(h_1,\ldots,h_N)$. Let us fix metrics $\lVert-\rVert_H$ and $\lVert-\rVert_Y$ on $H:=\mathbb{R}^N$ and $Y:=\op{Fun}(\SX,\mathbb{R}^m)\cong \mathbb{R}^{mN}$ respectively. Then we define the total variation of $f$ over $\SX$ by a perturbation $\bar{h}=(h_1,\ldots,h_N)\in H$ on variable $X_j$ as
\begin{multline*}
    v_{\SX,j}(f,\bar{h})=\\
    \lVert \left(f(x_{i,1},\ldots,x_{i,j}+h_{i},\ldots, x_{i,n}) - f(x_{i,1},\ldots,x_{i,n}) \right)_{i=1}^N \rVert_Y.
\end{multline*}

We define the \emph{sensitivity of $f$ with respect to variable $X_j$ over the dataset $\SX$ for the metrics $\lVert-\rVert_H$ and $\lVert-\rVert_Y$} as the maximum variation $v_{\SX,j}(f,\bar{h})$ relative to the size of small perturbations $\bar{h}$.
$$s_{\SX,j}(f):=\lim_{\varepsilon \to 0} \frac{ \sup_{\lVert \bar{h} \rVert_H=\varepsilon} v_{\SX,j}(f,\bar{h})}{\varepsilon}.$$

A natural setup for this metric analysis is to choose the involved metrics to be $L^p$ norms. Recall that for each $p\in [1,\infty)$, we define the $L^p$ norm as
$$\lVert (x_1,\ldots,x_M) \rVert_p=\left(\sum_{i=1}^M |x_i|^p \right)^{1/p}\, .$$
Taking the limit when $p\to\infty$, we also have
$$\lVert (x_1,\ldots,x_M)\rVert_\infty=\max\{|x_i|\}\, .$$
In this case, explicit formulas for $s_{\SX,j}(f)$ can be computed in terms of the differential of $f$ at each point in $\SX$.

Let $d_jf$ denote the differential of $f$ with respect to variable $X_j$, i.e., if $f=(f_1,\ldots,f_m)$, then
$$d_jf= \left( \frac{\partial f_1}{\partial X_j} dX_j, \ldots, \frac{\partial f_m}{\partial X_j} dX_j\right)\, .$$

Then the following theorem (whose complete proof can be found at \hyperref[annex:proofs]{Annex I}) enables an efficient and simple computation of the metric sensitivity invariant $s_{\SX,j}$ when the chosen metrics are $L^p$ metrics.

\begin{theorem}
\label{thm:mainTheorem}
Let $f=(f_1,\ldots,f_m):\mathbb{R}^n\longrightarrow \mathbb{R}^m$ be a $\SC^2$ function. Assume that $\lVert - \rVert_H=\lVert - \rVert_p$ and $\lVert - \rVert_Y=\lVert - \rVert_q$.
\begin{itemize}
    \item If $p\le q$ then
    $$s_{\SX,j}(f)=\max_{i}\left \{ \left \lVert d_jf(\bar{x}_i) \right \rVert_q \right\} .$$
    \item Otherwise, if $p>q$ then
    $$s_{\SX,j}(f)=\left (\sum_{i=1}^N \left( \sum_{k=1}^m \left |\frac{\partial f_k}{\partial X_j}(\bar{x}_i)\right |^q \right)^{\frac{p}{p-q}} \right )^{\frac{p-q}{pq}}.$$
\end{itemize}
\end{theorem}

\begin{remark}
Observe that when we take $L^p$ and $L^q$ norms on $H$ and $Y$ respectively, then $s_{\SX,j}(f)$ is a norm.
\end{remark}

\subsection{Sensitivity \texorpdfstring{$\alpha$}{alpha}-curves associated to a real function}
\label{section:alpha-curve}

Some interesting additional analysis can be derived from Theorem \ref{thm:mainTheorem} when the function $f$ is scalar. Let $f:\mathbb{R}^n\longrightarrow \mathbb{R}$ be a scalar function and let $\SX=\{\bar{x}_i\}_{i=1}^N$ be a dataset with $\bar{x}_i\in \mathbb{R}^n$. The previous Theorem allows us to compute explicitly the sensitivity $s_{\SX,j}(f)$ for each choice of $L^p$ norms on the perturbation and the target values. We observe, however, that when the target of the function is $\mathbb{R}$, some of the sensitivities agree for different choices of $(p,q)$, resulting in the fact that all the $L^p$ norm choices can be summarized on a $1$-parametric set of metrics which can then be rewritten in terms of the $\alpha$-mean of the values $\left | \frac{\partial f}{\partial X_j}(\bar{x}_i)\right |$ when $\bar{x}_i$ runs through the dataset $\SX$ and $1\le \alpha\le \infty$.

\begin{corollary}
\label{cor:alpha-curve}
If $f:\mathbb{R}^n\longrightarrow \mathbb{R}$ is a $\SC^2$ function, $\lVert - \rVert_H=\lVert - \rVert_p$ and $\lVert - \rVert_Y=\lVert - \rVert_q$ with $p> q$, then
$$s_{\SX,j}(f)=N^{1/\alpha} M_{\alpha}\left \{ \left | \frac{\partial f}{\partial X_j}(\bar{x}_i)\right | \right\},$$
where $\alpha=\frac{pq}{p-q}$ and
$$M_{\alpha}\{t_1,\ldots,t_N\}=\left( \frac{\sum_{i=1}^N t_i^\alpha}{N} \right)^{1/\alpha}$$
is the generalized $\alpha$-mean of the values. When $p\le q$ then
$$s_{\SX,j}(f)=M_{\infty} \left \{ \left | \frac{\partial f}{\partial X_j}(\bar{x}_i)\right | \right\} = \max_i \left \{ \left | \frac{\partial f}{\partial X_j}(\bar{x}_i)\right | \right\}.$$
\end{corollary}

This motivates the following definition. Let us define the \emph{$\alpha$-mean sensitivity of $f$ with respect to variable $X_j$ on the dataset $\SX$} as
$$\op{ms}_{\SX,j}^\alpha(f) := M_{\alpha}\left \{ \left | \frac{\partial f}{\partial X_j}(\bar{x}_i)\right | \right\}.$$
Then, define the sensitivity $\alpha$-curve as the map
$$\begin{array}{c}
\quad \op{ms}_{\SX,j}(f) : [1,\infty] \longrightarrow [0,\infty) \quad\\
\quad \quad \quad \quad \quad \; \; \; \; \; \alpha \; \; \; \mapsto  \; \op{ms}_{\SX,j}^\alpha(f)
\end{array}.$$

On the other hand, observe that the Generalized Mean Inequality implies that for each $0\le \alpha<\beta \le \infty$ we have
$$M_{\alpha}\left \{ \left | \frac{\partial f}{\partial X_j}(\bar{x}_i)\right | \right\}\le M_{\beta}\left \{ \left | \frac{\partial f}{\partial X_j}(\bar{x}_i)\right | \right\}$$
and we know that 
$$M_\infty\left \{ \left | \frac{\partial f}{\partial X_j}(\bar{x}_i)\right | \right\}= \lim_{\alpha\to \infty} M_{\alpha}\left \{ \left | \frac{\partial f}{\partial X_j}(\bar{x}_i)\right | \right\}$$
so we conclude that $\op{ms}_{\SX,j}(f)$ is an increasing bounded curve whose limit when $\alpha\to \infty$ is $\op{ms}_{\SX,j}^\infty(f)$. In virtue of the metrical interpretation of the sensitivity from the previous section, a representation of this curve, together with the asymptotic value $\op{ms}_{\SX,j}^\infty(f)$, yields an interesting visualization of the whole sensitivity analysis which is independent on the choice of the $L^p$ norms on the perturbation and target spaces. We will call this representation the \emph{sensitivity $\alpha$-curve associated to $f$ with respect to variable $X_j$ over the dataset $\SX$}.

\subsection{Relation with derivative distributions}
\label{section:distributionsrelation}
The $\alpha$-curves from the previous section can be given an alternative description in terms of the distribution of partial derivatives \cite{pizarroso_2022} mentioned in the introduction. This duality reinforces the usefulness of the $\alpha$-mean sensitivities and the $\alpha$-curve as quantitative tools for performing deep meaningful sensitivity analysis.

Assume that the points of the dataset $\SX$ have been drawn randomly uniformly and that, therefore, they inherit a uniform discrete distribution on them (with probability $1/N$ over each point). Let us consider the function $g_j(x)=\left | \frac{\partial f}{\partial X_j}(x)\right |$ representing the (local) sensitivity of $f$ with respect to $X_j$ at the point $x$. Then a direct computation shows that for each $\alpha\in [1,\infty)$:
$$\mathbb{E}[g_j^\alpha] = \left(\op{ms}_{\SX,j}^\alpha(f)\right)^\alpha.$$
As a consequence, all moment maps of the distributions of partial derivatives $g_j$ can be computed as polynomials in the $\alpha$-mean sensitivities.  In particular, this proofs that the $\alpha$-mean sensitivities encode exactly as much information as the moment maps of the distributions of partial derivatives across the dataset.

This dual interpretation has interesting theoretical ramifications on the interpretation and validity of several sensitivity analysis methodologies. On the one hand, it proves that any qualitative analysis on the distributions of local sensitivities can be aggregated as an analysis of the corresponding $\alpha$-curve instead, showing that $\alpha$-curves are at least as informative as the distribution of local sensitivities. Nevertheless, we will provide some experimental evidences which prove that an analysis on $\alpha$-curves allows an easier detection of certain qualitative properties of the dependence of a function with respect to a variable than the analysis on moment maps of the distributions of partial derivatives.

On the other hand, this allows to provide an additional metric interpretation (and, thus, additional theoretical support) for the usage of the moments of derivative distributions as sensitivity measures, as used in \cite{White_2001, munoz_1998, pizarroso_2022}. 

\section{Methodology of the \texorpdfstring{$\alpha$}{alpha}-curves analysis}
\label{section:methodology-of-alpha-curves}
As a consequence of the previous theoretical analysis, we have built a family of metrics capable of quantifying the sensitivity of any model $f$ with respect to a variable $X_j$. In this section, we will explain some methodologies capable of exploiting this new theoretical framework to detect some patterns in the roles that each variable play in the model. Contrary to other sensitivity analysis methodologies, our proposed method will be able to capture high sensitivity behaviors of a variable which may only occur in a certain local region of the phase space, even if that variable presents generally a low sensitivity across the rest of the dataset.

For simplicity, we will focus the analysis on scalar regression problems. In this case, we have proven that computing the $\alpha$-curves of the function for each variable allows a complete and simultaneous study of all the possible euclidean metric interpretations of sensitivity for each variable across the whole dataset. Analysing the differences between the values of the $\alpha$-sensitivities $\op{ms}_{\SX,j}^\alpha$ for different variables $X_j$ and different choices of $\alpha$ opens new deeper layers to the sensitivity analysis and can be used to detect further properties and interactions between the variables of a model than other methods used in the literature. We proceed to describe a methodology for retrieving some of the properties of the variables of a model from an $\alpha$-curve plot.

We would like to clarify that this is not an exhaustive description of the types of analysis that can be done within the $\alpha$-curve framework. For instance, we believe that this family of sensitivity metrics and the theoretical framework which supports them can be incorporated as part of more complex XAI analysis. It is just a showcase of some of its basic properties.

\subsection{Model sensitivity analysis and \texorpdfstring{$\alpha$}{alpha}-curve plots}
The $\alpha$-curve analysis presented in this work is a tool designed to study scalar regression data models (in particular, as stated in Theorem \ref{thm:mainTheorem} and Corollary \ref{cor:alpha-curve}, $\SC^2$ data models). Thus, in order to analyse raw data with this method, the first step is to build an appropriate representative data model $f$. An important example is to choose $f$ to be an AI system of some type trained over the data. The $\alpha$-sensitivity analysis is always meant to study properties of the chosen model $f$ and not necessarily on the data which generated it (eg., it will analyse the way a trained AI model interacts with its input variables, not the data which was used to train it).

In order to have proper comparisons between variables, it is recommended that the variables are normalized or have comparable magnitudes before constructing the model $f$ and performing the $\alpha$-sensitivity analysis. Otherwise, renormalisations may need to be taken into account when performing the study (the $\alpha$-curve equation is homogeneous of degree 1 on linear scaling of the functions and variables).

An $\alpha$-curve plot (see Figure \ref{fig:fig_sqrt_alpha_curves}) is a 2D plot in which we draw together, for each variable $X_j$, the variation of the $\alpha$-sensitivity of the model, $\op{ms}_{\SX,j}^\alpha(f)$, when $\alpha$ varies. By the Generalized Mean Inequality, we know that each $\alpha$-curve is increasing and bounded. As the curve can sometimes increase very slowly as $\alpha$ increases, and the limit value $\alpha=\infty$ is interesting for our analysis, we will draw the curve up to a certain limit (we found that $\alpha\in [1,16]$ is enough in most experiments) and then add to the plot the asymptotic value $\op{ms}_{\SX,j}^\infty(f)$ for each variable $X_j$. See Section \ref{section:experiments} for examples.

\subsection{Comparison of variables for a fixed \texorpdfstring{$\alpha$}{alpha}}
For each value of $\alpha$, the set of values $\op{ms}_{\SX,j}^\alpha(f)$ provides a sound measure of the sensitivity of the model $f$ with respect to each variable $X_j$. Thus, it can be used to compare which input variables are more relevant for the output in the model.

In the literature, the metrics obtained from $\alpha=1$ and $\alpha=2$ (or, derived ones, like the variance of the distribution of partial derivatives, which can be computed from these two values, cf. Section \ref{section:distributionsrelation}), have been used as sensitivity metrics \cite{White_2001, munoz_1998, pizarroso_2022} and utilized, for instance, for variable pruning \cite{yeh_cheng_2010, zeng_zhen_he_han_2017}. 

Corollary \ref{cor:alpha-curve} now provides additional sound theoretical framework which supports mathematically the choice of any of these  aggregation functions as a way to derive a total sensitivity metric from the values of the derivatives of the function across the dataset. 

On the other hand, each vertical cut $\alpha$ to the $\alpha$-curve plot can be used to compare the variables and draw quantitative and qualitative conclusions about their relative relevance to the model. If there are model-driven reasons to fix a certain metric on the input and perturbation spaces (see \ref{thm:mainTheorem}), then the corresponding $\alpha$-value should be chosen for the comparison. Otherwise, any value of $\alpha$ could, theoretically, be used for the sensitivity comparison task independently. Nevertheless, analyzing the whole picture across all different $\alpha$ allows a deeper understanding of the behavior of the model.

Due to the properties of $\alpha$-means, as $\alpha$ increases, the average value $\op{ms}_{\SX,j}^\alpha(f)$ takes more into account the existence of regions in the dataset where the sensitivity with respect to the variable is higher than average ("exceptionally sensitive" or "localized high sensitivity" behavior). Lower values of $\alpha$ focus instead on the "average behavior" of the function with respect to the variable.

The analysis of high values of $\alpha$ may be crucial for certain tasks like variable pruning. It is possible that a variable has almost no impact on a regression problem if one looks at a generic point in the phase space, but that there exists a mode change in the model making the variable very relevant for the analysis when the inputs move inside a certain critical region (think, for instance, in the case where there exist "activation" variables or states, which enable a different variable to influence the result but otherwise disable it). A general pruning analysis with $\alpha=1$ or $\alpha=2$ could "discard" the variable as irrelevant for the model, whereas it might be the most relevant variable for high $\alpha$ metrics (see, for instance, example .\ref{subsubsec:sqrt})

The limit values $\op{ms}_{\SX,j}^\infty(f)$ included in the plot help identify the extremal cases. They measure the maximum sensitivity of $f$ with respect to the input variable $X_j$ that can be found at any point in the dataset.
\begin{figure*}
    \centering
    \begin{subfigure}{.3\linewidth}
        \centering
        \includegraphics[width=\linewidth, height = 0.9\linewidth]{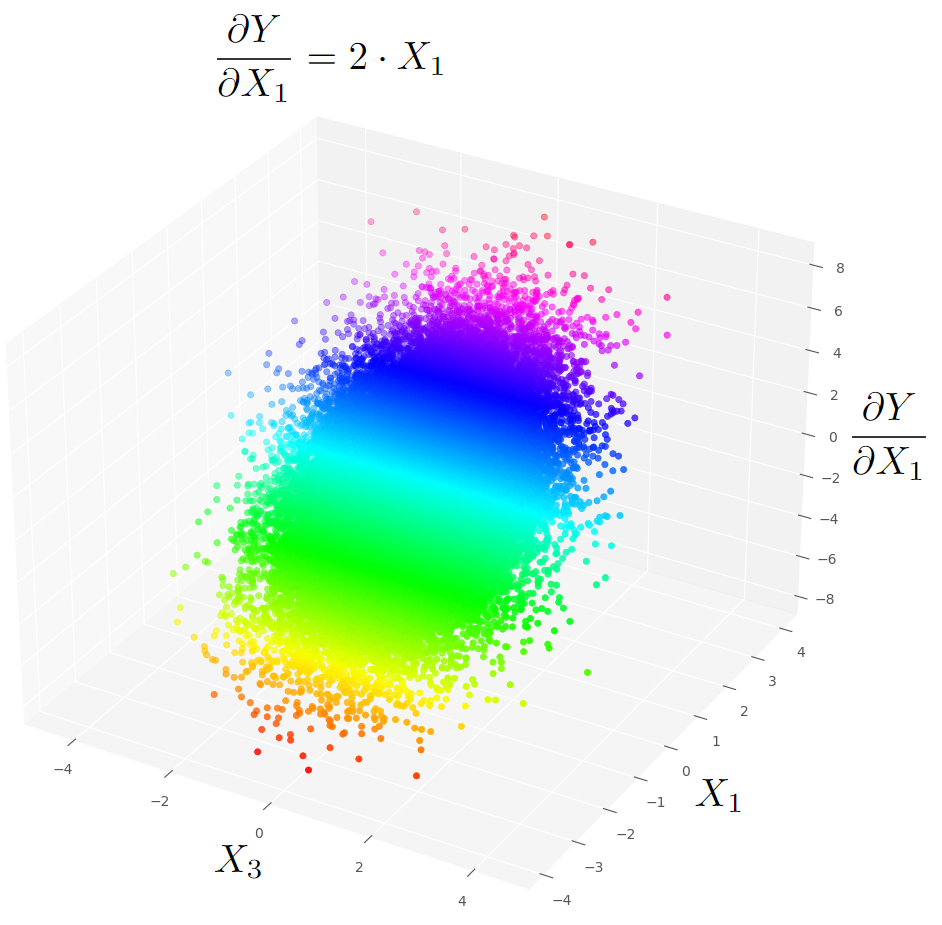}
        \caption{\label{fig:fig_derx1_sqrt}}
        \end{subfigure}
    \begin{subfigure}{.3\linewidth}
        \centering
        \includegraphics[width=\linewidth, height = 0.9\linewidth]{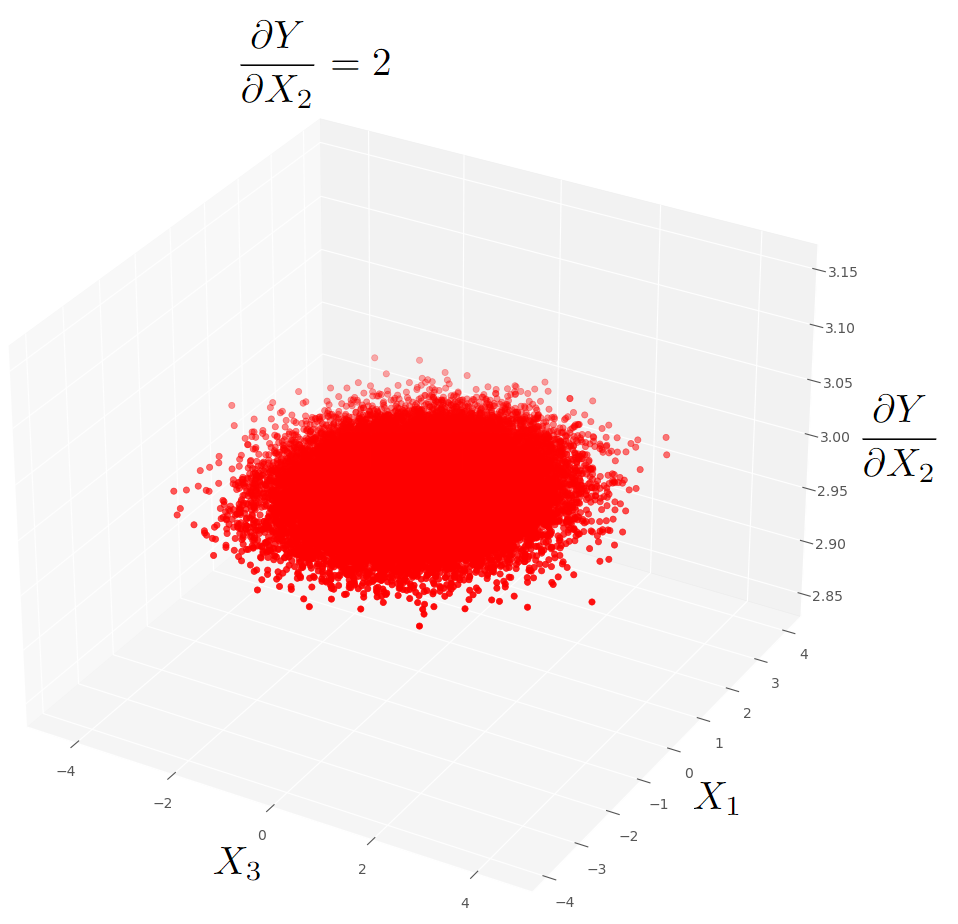}
        \caption{\label{fig:fig_derx2_sqrt}}
    \end{subfigure}
    \begin{subfigure}{.3\linewidth}
        \centering
        \includegraphics[width=\linewidth, height = 0.9\linewidth]{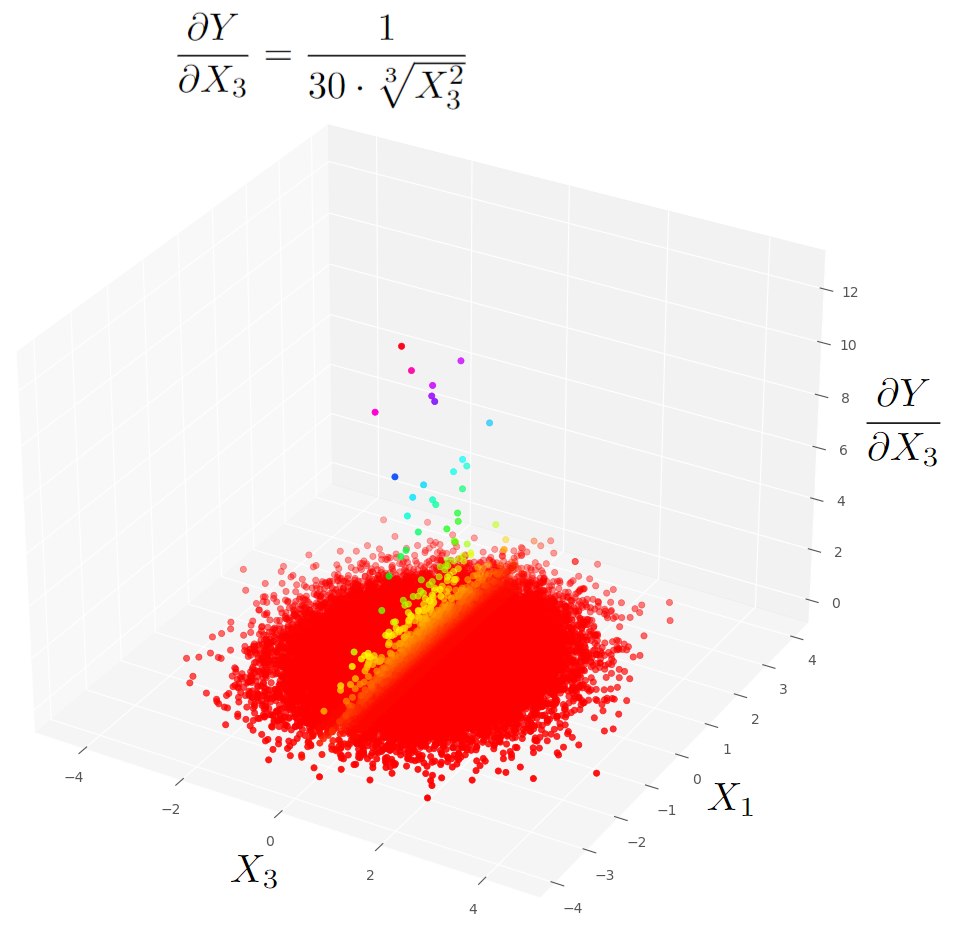}
        \caption{\label{fig:fig_derx3_sqrt}}
    \end{subfigure}
    \caption{\label{fig:fig_der_3d_sqrt} 3D plots of partial derivatives of output $Y$ with respect to inputs $X_1$, $X_2$ and $X_3$ ((a), (b) and (c) respectively) for square root synthetic dataset. X-axis follows $X_1$ and y-axis follows $X_3$ in the three plots. $X_2$ is not used as plot axis due to the irrelevance of this variable on the derivative plots.}
\end{figure*}

\subsection{Analysis of the variation of an \texorpdfstring{$\alpha$}{alpha}-curve}

Due to the aforementioned properties of the $\alpha$-means (consequence of the convexity of the power functions for exponents at least 1), studying the variation of the $\alpha$-sensitivity when $\alpha$ changes can give a lot of information on the dynamics of the variables of the model $f$. Let us study some examples.

\subsubsection{Linearity analysis}
By the Generalized Mean Inequality, the $\alpha$-curve of variable $X_j$ is constant if and only if $f$ is of the form
$$f(X_1,\ldots,X_n)=g(X_1,\ldots,X_{j-1},X_{j+1},\ldots,X_n)+CX_j$$
for some function $g$ depending only on the rest of the variables.
By extension, the closer an $\alpha$-curve is to be flat, the closer the dependence of $f$ with respect to $X_j$ is to a linear dependence. For instance, when an $\alpha$-curve starts almost flat and then starts increasing more starting at some alpha, this can mean that the derivative $\frac{\partial f}{\partial X_j}$ has a low variation through the majority of the dataset, but that there are one or more regions of the phase space where it changes more, either due to its own non-linear behavior (like an activation function, or function where the derivative increases close to a point, like a $\SC^2$-approximation of a square root), or due to an interaction with other variables.

\subsubsection{Irrelevant variables}
As a particular case of the previous analysis, $f$ does not depend on a variable $X_j$ if and only if the $\alpha$-curve is constantly zero. The closer a curve is to 0, the less important the variable is for the model.

If a curve starts very constant and close to 0 but increases afterwards, this indicates that the output of the model has, in general, a low dependence on the variable, but that there exists a region in the phase space in which the variable is indeed relevant for the model.

These properties can be used to improve the specificity of variable pruning methodologies. If a variable presents a low value of $\op{ms}_{\SX,j}^\infty(f)$ (and, thus, the whole $\alpha$-curve is low) then it is not important for the model and it can be safely removed. On the contrary, a variable presenting higher values of the curve for some $\alpha$ (and thus, a higher $\op{ms}_{\SX,j}^\infty(f)$) should not be pruned without a deeper analysis.

\subsubsection{Detection of local regions with high sensitivity}
As outlined before, it is possible for a variable to have low sensitivity for low $\alpha$ but high sensitivity in higher $\alpha$. This makes comparing its $\alpha$-sensitivity with the $\alpha$-sensitivity of other variables depend heavily on $\alpha$. When this happens and a variable is not sensitive for low $\alpha$ but it becomes highly sensitive for high $\alpha$, two things can happen.
\begin{itemize}
    \item The variable shows a non-linear behavior on $X_j$ which makes the partial derivative $\frac{\partial f}{\partial X_j}$ increase only on certain values of $X_j$.
    \item There exists an interaction between the variable and a combination of other variables which makes the derivative become high in a certain region of the phase space.
\end{itemize}
The higher the variation of the $\alpha$-curve and the earlier this variation appears, the stronger and more generalized the interaction or non-linear effect is across the dataset. If the $\alpha$-curve starts flat and then there is a sudden increase, it is more probable that the interaction or non-linear input effect on the output is relevant only in certain bounded areas of the dataset.

A limitation of this method (see Section \ref{section:conclusions}) is that it is difficult to distinguish between the increase in sensitivity produced by an interaction between variables (eg. when they are equally distributed) and a non-linear input effect (which can be thought of as a self-interaction of the variable). We expect to solve this limitation in future work through the usage of complementary interaction analysis methodologies.

\section{Experimental results}
\label{section:experiments}

This section contains XAI analysis performed on various synthetic and real datasets using sensitivity analysis based on partial derivatives, SHAP, input permutation and the $\alpha$-curves method.

\subsection{Synthetic dataset}

A synthetic dataset with known derivatives is used to illustrate the usefulness of the $\alpha$-curves to retrieve information about how the model uses the input variables to predict the output variable. The dataset is composed by 8 input variables $[X_{1}, \ldots, X_{8}]$  and one output variable $Y \in \mathbb{R}$ created as a function of the input variables, i.e., $Y = f\left(\mathbf{X}\right)$. Input variables $\mathbf{X}$ are 50000 samples drawn from a normal distribution with $\mu = 0$ and $\sigma = 1$. Partial derivatives and SHAP values are calculated analytically from the output expression for each dataset to avoid inherent modelling error which might obfuscate relationships between inputs and outputs. 

\begin{figure*}
    \centering
    \begin{subfigure}{.45\linewidth}
        \centering
        \includegraphics[width=\linewidth, height=\linewidth]{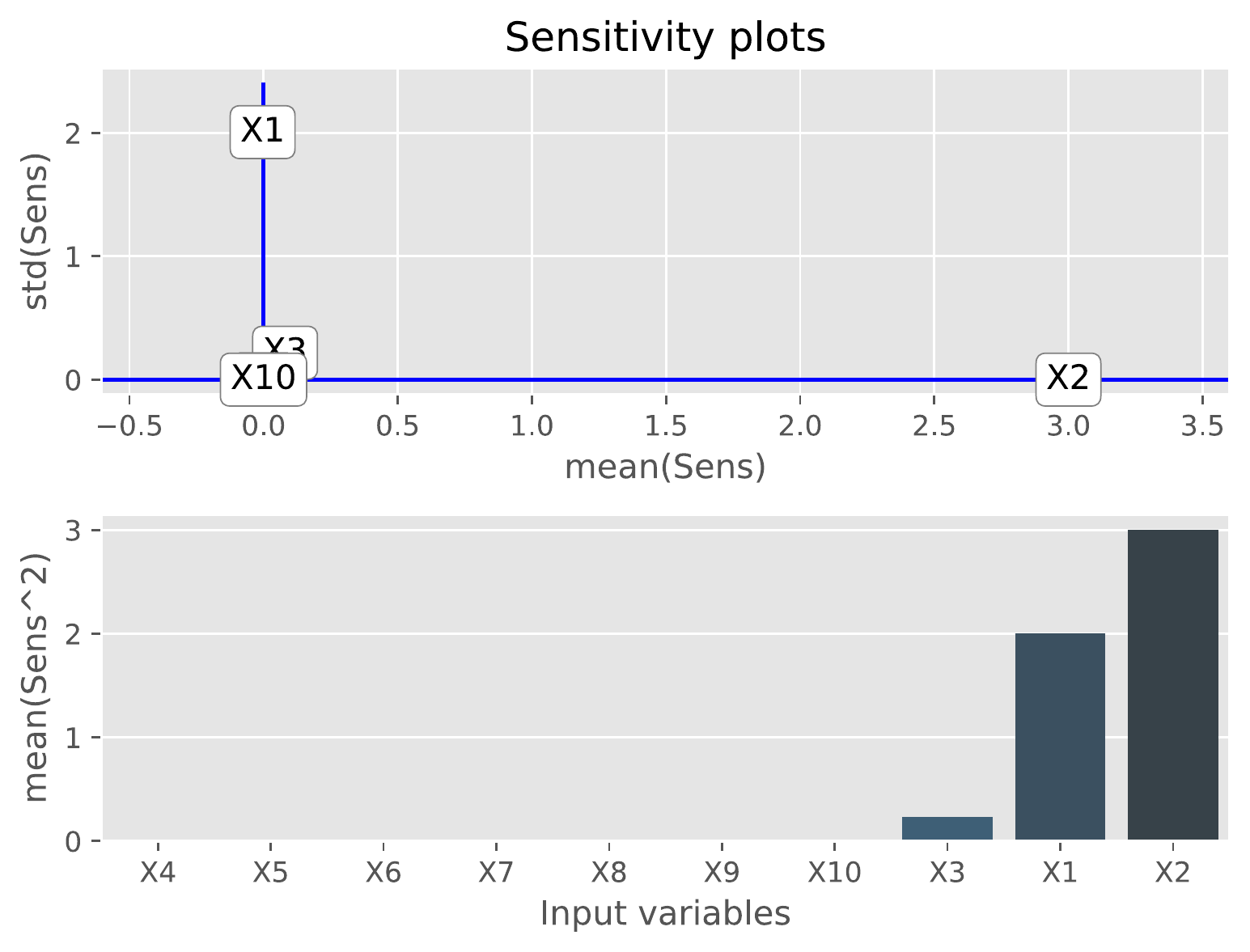}
        \caption{\label{fig:fig_sqrt_sens_plots} Sensitivity plots of cubic root synthetic dataset.}
        \end{subfigure}
    \begin{subfigure}{.45\linewidth}
        \centering
        \includegraphics[width=\linewidth, height=\linewidth]{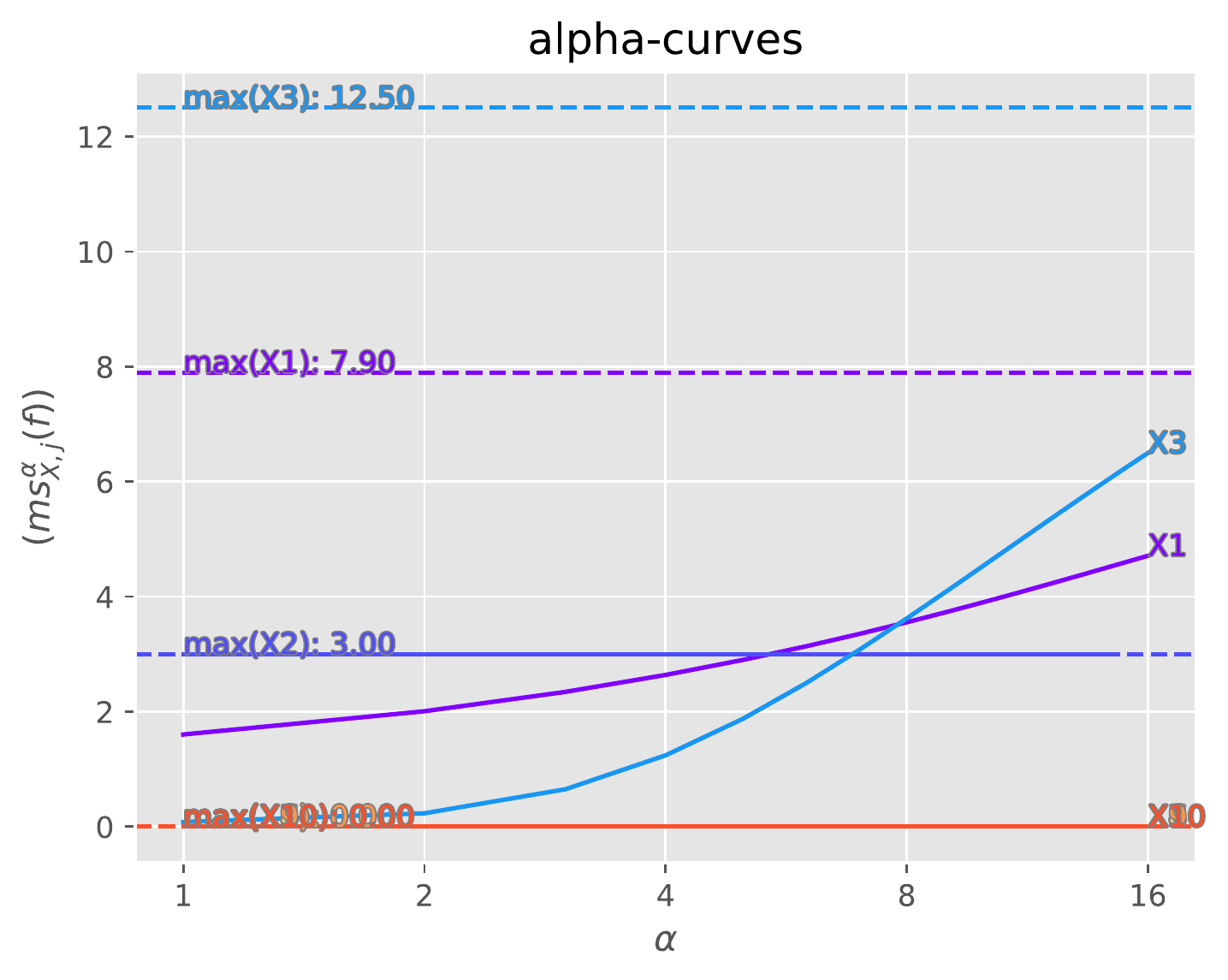}
        \caption{\label{fig:fig_sqrt_alpha_curves} $\alpha$-curves of cubic root synthetic dataset.}
    \end{subfigure}
    \begin{subfigure}{.45\linewidth}
        \centering
        \includegraphics[width=\linewidth, height=\linewidth]{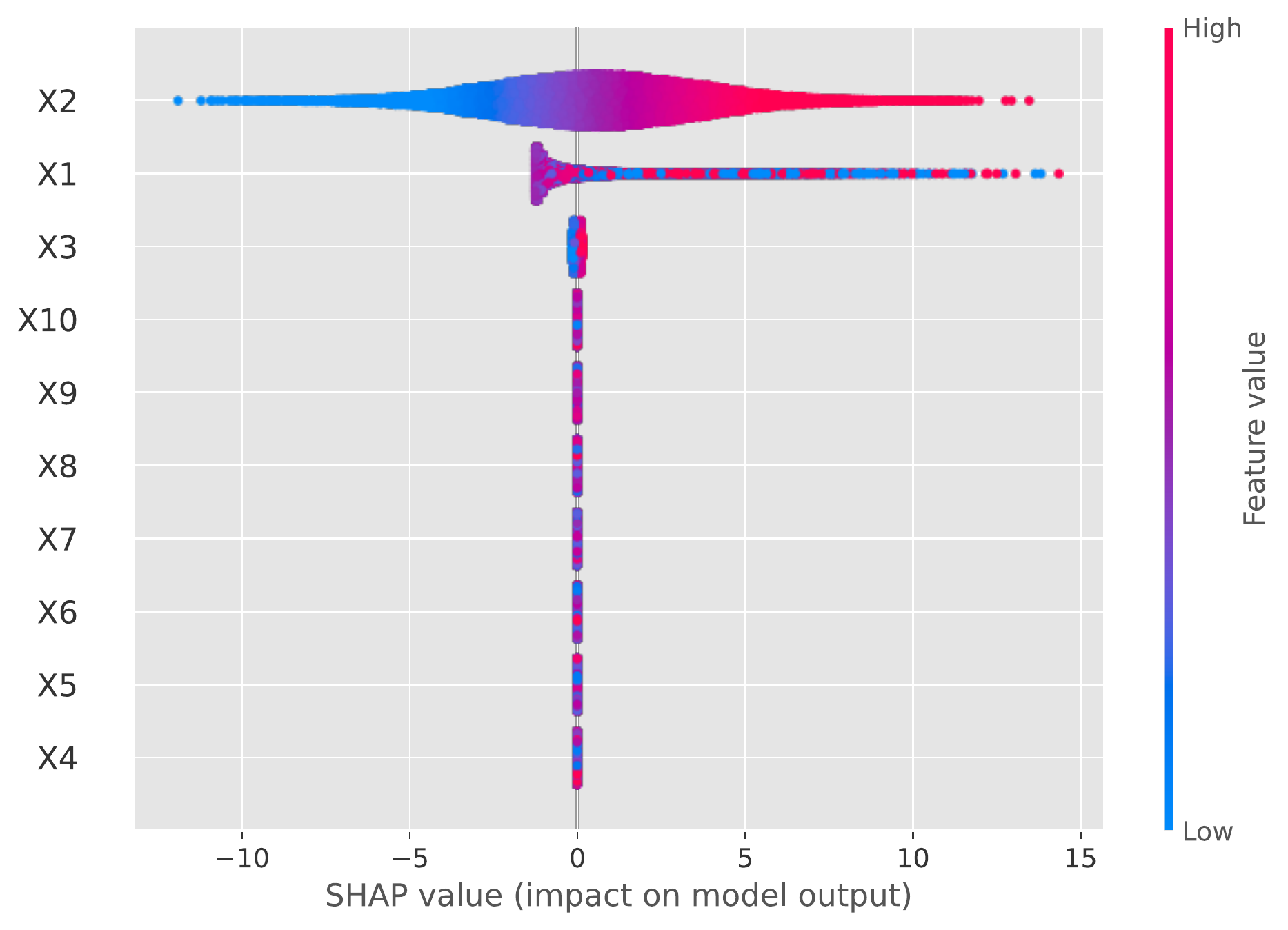}
        \caption{\label{fig:fig_sqrt_shap_values} SHAP values of cubic root synthetic dataset.}
    \end{subfigure}
    \begin{subfigure}{.45\linewidth}
        \centering
        \includegraphics[width=\linewidth, height=\linewidth]{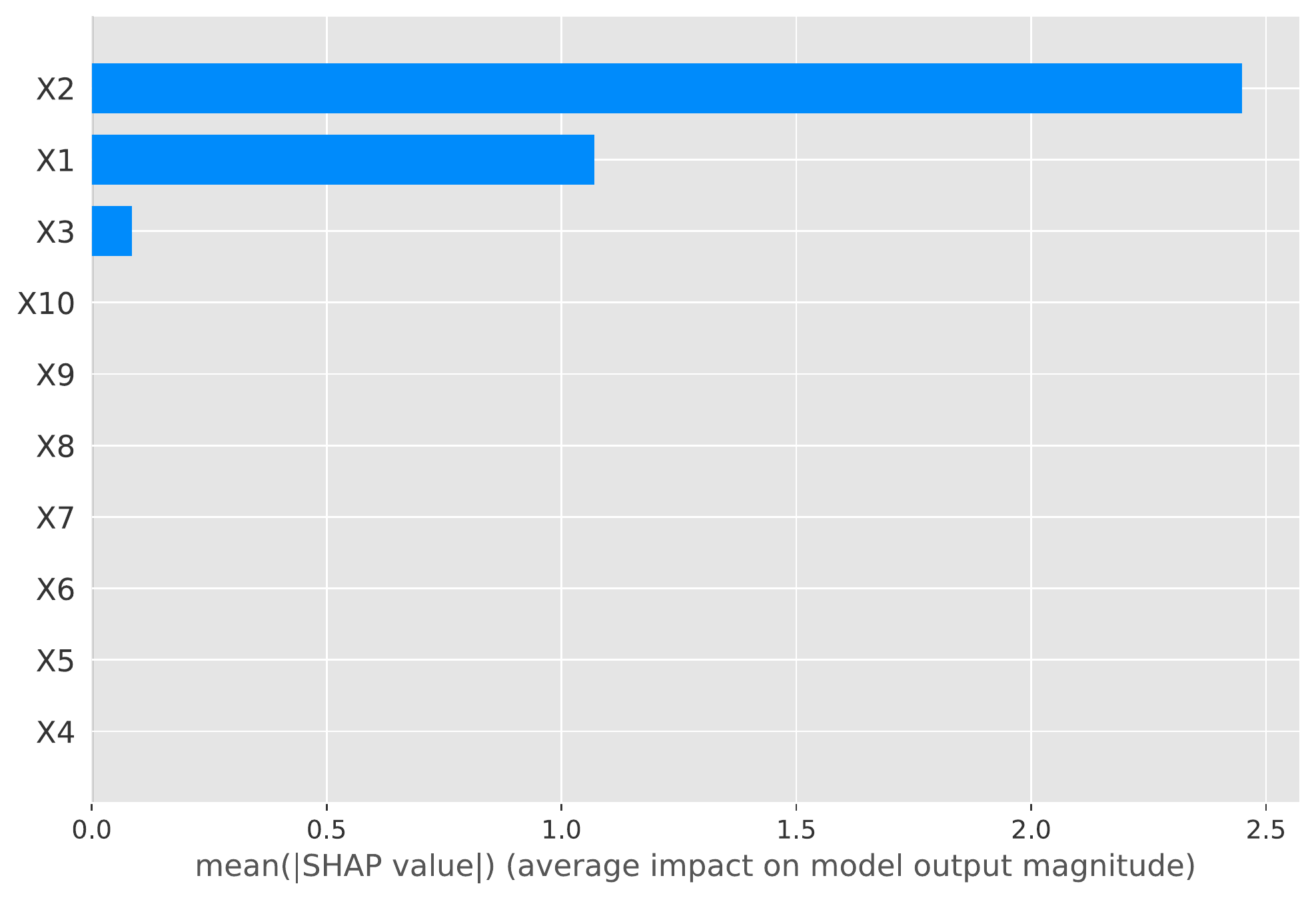}
        \caption{\label{fig:fig_sqrt_shap_importance} SHAP Importance values of cubic root synthetic dataset.}
    \end{subfigure}
    \begin{subfigure}{.45\linewidth}
        \centering
        \includegraphics[width=\linewidth, height=0.8\linewidth]{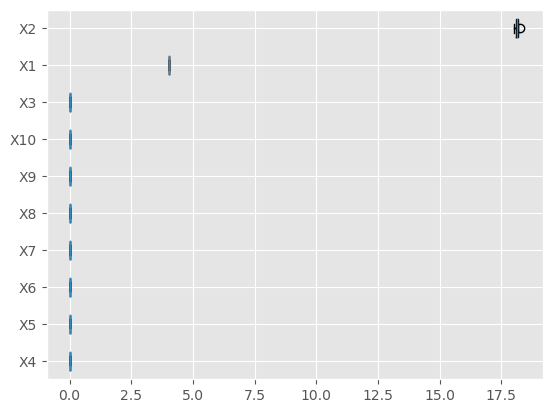}
        \caption{\label{fig:fig_sqrt_input_permutation} Input permutation importance values of cubic root synthetic dataset.}
    \end{subfigure}
    \caption{\label{fig:fig_sqrt_xai} XAI techniques plots of cubic root synthetic dataset.}
\end{figure*}

\subsubsection{Cubic Root}
\label{subsubsec:sqrt}
In this case, the output follows the next expression:

\begin{equation}
\label{eq_syn_squareroot}
    Y = (X_1)^2 + 2 \cdot X_2 + \frac{1}{10} \cdot \sqrt[3]{X_3}\, .
\end{equation}

From figure \ref{fig:fig_der_3d_sqrt}, we can conclude that $X_2$ has a linear relationship with $Y$ as $\frac{\partial Y}{\partial X_2}$ is constant and different from zero for all samples. $X_1$ and $X_3$ have a non-linear relationship with $Y$, as $\frac{\partial Y}{\partial X_1}$ and $\frac{\partial Y}{\partial X_3}$ are not constant for all samples. Furthermore, figure \ref{fig:fig_derx3_sqrt} shows that $\frac{\partial Y}{\partial X_3} = 0$ for most samples, except for the samples where $X_3$ is close to $0$. In these samples, sensitivities with respect of $X_3$ are far higher than with respect to the other input variables, so changes of $X_3$ in this region of the input space shall provoke large changes on $Y$. This can be understood as a local importance of $X_3$, and it shall be detected by XAI methods.

Results of XAI analysis performed on equation \ref{eq_syn_squareroot} are presented in figure \ref{fig:fig_sqrt_xai}. 

Figure \ref{fig:fig_sqrt_sens_plots} shows the sensitivity plots as introduced in \cite{pizarroso_2022}. First plot shows two sensitivity metrics: mean (x-axis) and standard deviation (y-axis). Second plot of figure \ref{fig:fig_sqrt_sens_plots} shows the mean squared sensitivity for each of the input variables, which could be used as a variable importance metric. A broader explanation of these metrics can be found in section \ref{section:sota}. According to this metrics, the following information can be retrieved from figure \ref{fig:fig_sqrt_sens_plots}:

\begin{itemize}
    \item $X_2$ variable has a linear relationship with the output.
    \item $X_1$ variable has a non-linear relationship with the output. 
    \item $X_3$ is almost irrelevant to predict the output, with much lower importance than $X_1$ and $X_2$, but greater than $X_4-X_8$.
    \item The remaining variables have no relationship with the output.
\end{itemize}

The same information is obtained using SHAP from figures \ref{fig:fig_sqrt_shap_values} and \ref{fig:fig_sqrt_shap_importance}. In figure \ref{fig:fig_sqrt_shap_values}, linear relationship of $Y$ and input $X_2$ can be seen in the perfect correlation between the values of $X_2$ and its impact on $Y$. Non-linear relationships of $Y$ and inputs $X_1$ is also easily detected, as there is no correlation between the values of $X_1$ and its impact on $Y$. Figure \ref{fig:fig_sqrt_shap_importance} is analogous to second plot of figure \ref{fig:fig_sqrt_sens_plots}, but showing mean of the absolute SHAP values instead of mean squared sensitivity as variable importance measure. 

Figure \ref{fig:fig_sqrt_input_permutation} shows the importance metrics assigned to the input variable with the input permutation technique. These metrics are almost identical to SHAP importance metrics presented in figure \ref{fig:fig_sqrt_shap_importance}, with similar relative importances between input variables. This technique does not provide information about the type of relationship between output and input variables, but it is notably less computationally expensive than the others. 

Using the $\alpha-$curves methodology described in section \ref{section:methodology-of-alpha-curves}, previously obtained information can be obtained from figure \ref{fig:fig_sqrt_alpha_curves}. However, it also shows that, apart from the non-linearities presented in $X_1$ and $X_3$, there are regions where output $Y$ is far more sensitive to $X_3$ than to $X_2$. In fact, peak sensitivities in some samples are detected, as can be seen by the rapid increase of $\op{ms}^{\alpha}_{\SX,3}(f)$ for $\alpha > 4$. This information could not be retrieved using the other two methods, which assigned little importance to $X_3$, due to the aggregation techniques used to calculate importance of the input variables.

\begin{figure*}
    \centering
    \begin{subfigure}{.45\linewidth}
        \centering
        \includegraphics[width=\linewidth, height=\linewidth]{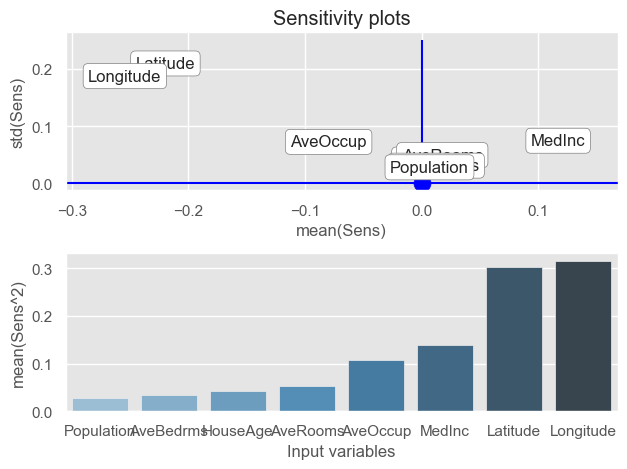}
        \caption{\label{fig:fig_ames_sens_plots} Sensitivity plots of california housing dataset.}
        \end{subfigure}
    \begin{subfigure}{.45\linewidth}
        \centering
        \includegraphics[width=\linewidth, height=\linewidth]{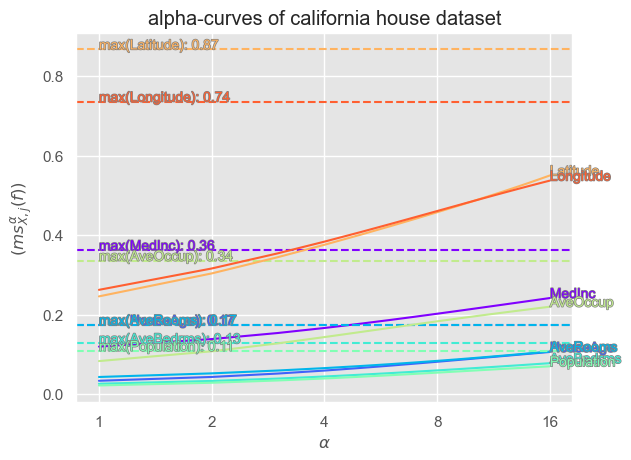}
        \caption{\label{fig:fig_ames_alpha_curves} $\alpha$-curves of california housing dataset.}
    \end{subfigure}
    \begin{subfigure}{.45\linewidth}
        \centering
        \includegraphics[width=\linewidth, height=0.7\linewidth]{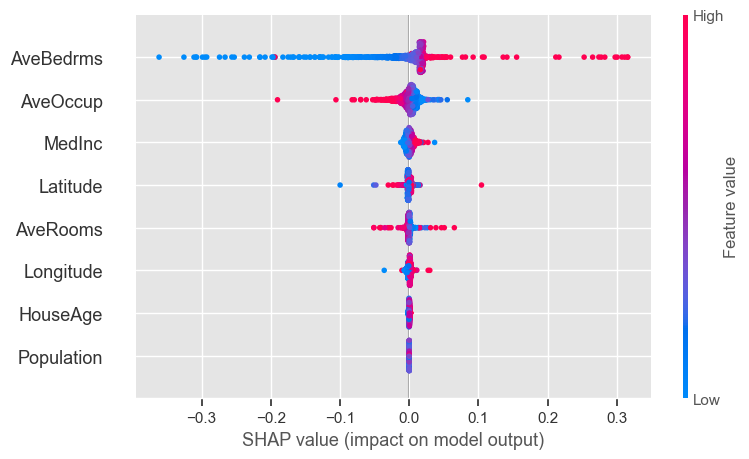}
        \caption{\label{fig:fig_ames_shap_values} SHAP values of california housing dataset.}
    \end{subfigure}
    \begin{subfigure}{.45\linewidth}
        \centering
        \includegraphics[width=\linewidth, height=0.7\linewidth]{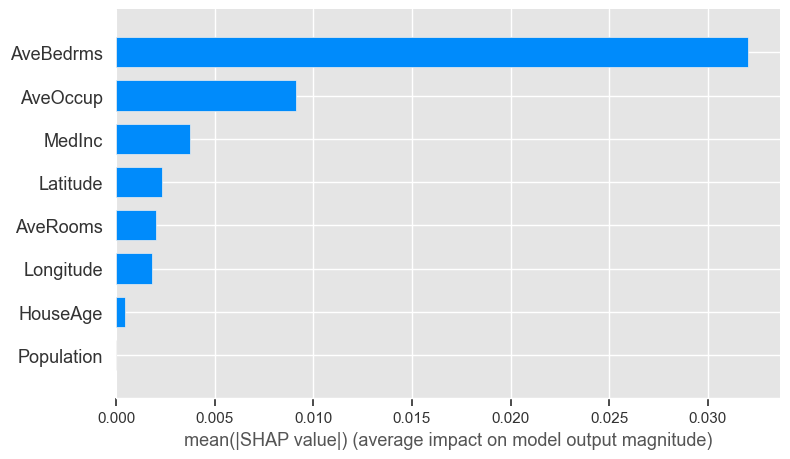}
        \caption{\label{fig:fig_ames_shap_importance} SHAP Importance values of california housing dataset.}
    \end{subfigure}
    \begin{subfigure}{.45\linewidth}
        \centering
        \includegraphics[width=\linewidth, height=0.8\linewidth]{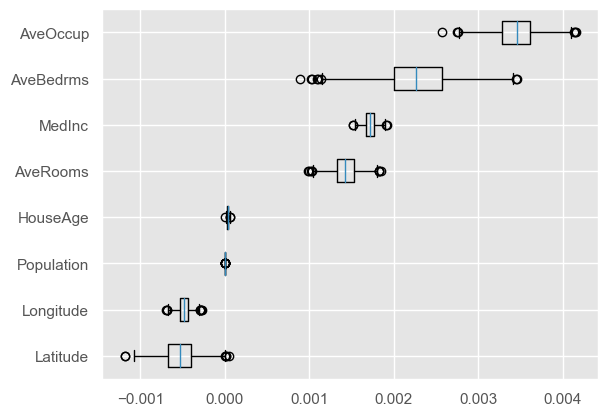}
        \caption{\label{fig:fig_ames_input_permutation} Input permutation importance values of california housing dataset.}
    \end{subfigure}
    \caption{\label{fig:fig_ames_xai} XAI techniques plots of california housing dataset. Techniques used are sensitivity analysis based on partial derivatives and SHAP.}
\end{figure*}

\subsection{Real datasets}

Based on the analysis performed in the previous section, similar analysis can be performed on datasets from real sources. In this section, sensitivity analysis are performed on the California housing \cite{california_2022} and Diabetes \cite{diabetes_sklearn} regression datasets. As the ground-truth of the relationship between input and output variables are not known, a Multi-Layer Perceptron (MLP) model with one hidden layer is trained to predict the output of each dataset. To facilitate optimization (and without any loss of generality), we standardize all features independently for all datasets. Target variables are rescaled such that they are between zero and one.  

The first partial derivatives of the MLP model are calculated using the $neuralsens$ package \cite{pizarroso_2022}. It must be noted that the $\alpha-$curves methodology are directly applied to the calculated partial derivatives, so the user-preferred package to calculate partial derivatives of a ML model can be used. The homonym package $SHAP$ \cite{Lundberg_2017} is used to calculate and analyze the SHAP values for each experiment. 

\subsubsection{California Housing}
This dataset was derived from the 1990 U.S. census, using one row per census block group. A block group is the smallest geographical unit for which the U.S. Census Bureau publishes sample data (a block group typically has a population of 600 to 3,000 people). The target variable is the median house value for California districts, expressed in hundreds of thousands of dollars (\$100,000) \cite{Kelley_1997}.

This dataset is composed of the following variables:
\begin{itemize}
    \item MedInc: median income in block group
    \item HouseAge: median house age in block group
    \item AveRooms: average number of rooms per household
    \item AveBedrms: average number of bedrooms per household
    \item Population: block group population
    \item AveOccup: average number of household members
    \item Latitude: block group latitude
    \item Longitude: block group longitude
    \item MedHouseVal: median price of block group
\end{itemize}

Following the methodology for sensitivity analysis, figure \ref{fig:fig_ames_sens_plots} shows that longitude and latitude of the block group are the most important variables with a highly non-linear relationship with the output, followed by the median income of the families in block group and the average number of household members with a more linear relationship. Rest of the variables have almost no relationship with the output and may be discarded from the model. This information seems intuitive, as the location of the house is usually one of the key variables to determine the value of a house. Also, number and type of neighbours in the area are usually a good indicator of the economy of the block, where fewer people with higher income might indicate exclusive villas and more people with lower income might indicate residential blocks. Rest of the variables correlates with what we can expect to influence the price of the house, this is, size of the house ($AveRooms$) and how up-to-date house features are ($HouseAge$) is more important than population of the block. One might object that $AveBedrms$ also correlates with the size of the house and shall be assigned a higher importance. However, size of the house can be determined with $AveRooms$, so the information provided by $AveBedrms$ may only be used to distinguish between same size houses with different number of bedrooms. This information might not influence house price as much as the house size, so consequently the $AveBedrms$ variable is not as important as $AveRooms$. 

More information can be retrieved using the $\alpha-$curves methodology. The location of the house is still the most important information to predict the price, but figure \ref{fig:fig_ames_alpha_curves} shows that, although $Latitude$ have a lower mean sensitivity, there are regions of the input dataset where the $Latitude$ sensitivity is the highest. This might indicate a region where a change in latitude mean changing between two blocks with great price differences, maybe between first and second beach line blocks. Rest of the variables shows similar information than the retrieved from figure \ref{fig:fig_ames_sens_plots}, although it can be seen than the maximum sensitivity of $HouseAge$ is similar to the maximum sensitivity of $AveRooms$. This might indicate that, although the mean effect of the $HouseAge$ is not as important as the $AveRooms$ variable, the house age might influence the price of the house as much as the size if the house. This makes sense, as an older house usually needs house renovations and this might decrease the price.

\begin{figure}[t!]
    \centering
    \includegraphics[width=1\linewidth, height = 0.8\linewidth]{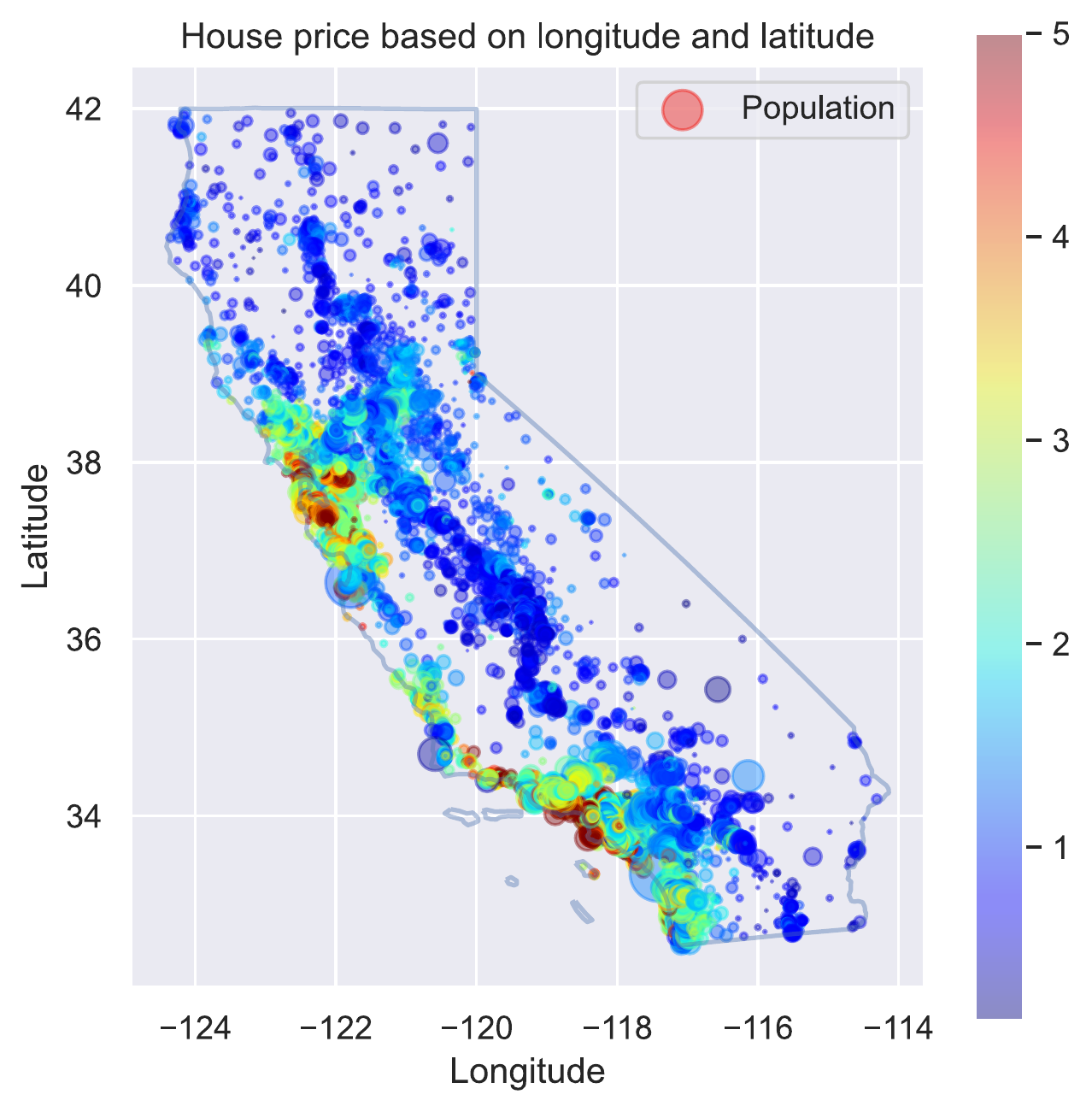}
    \caption{\label{fig:fig_heatmap_ames} Heatmap of median house price plotted in latitude and longitude. Color is given by median house price in block group, size is given by block group population.}
\end{figure}

It must be noted that it was not computationally feasible to analyze all the samples of the dataset using SHAP, so only 1000 random samples were analyzed. Figure \ref{fig:fig_ames_shap_values} and \ref{fig:fig_ames_shap_importance} shows that the most important variable is the average number of bedrooms per household followed by the average number of household members. Relationship between these inputs and the output seems linear, where a large number of bedrooms and fewer households corresponds to a higher house price. This correlates with the idea of luxury villas and residential blocks stated earlier. However, it appears counter-intuitive that the location of the block given by the $Latitude$ and $Longitude$ variable is barely important compared with the rest of variables. 

Considering the input permutation importances shown in figure \ref{fig:fig_ames_input_permutation}, it can be seen how the importances assigned to the input variables vary depending on the permutation performed on the variable. In this case, age of the house $HouseAge$ and population of the block $Population$ does not influence on the performance of the model (importance of these variables being zero is related to no change in the error metric when this variables are permutted), coinciding with the rest of the XAI techniques. House size occupation related variables are the most important according to this technique, assigning the lowest importances to the location related variables ($Longitude$ and $Latitude$ variables). Moreover, this technique assigns a negative importance to these variables, implying that permutting $Longitude$ and $Latitude$ results in a more accurate model. Again, this seems counter intuitive and misleading, and may be the result of how the variables were permutted.  

Figure \ref{fig:fig_heatmap_ames} shows the distribution of prices in California state, where we can see that based on latitude and longitude great differences in house price can be distinguished. As expected, the highest house prices are in the population centers on the beachfront (in this case, San Francisco and Los Angeles). This corroborates the information provided by the $\alpha-$curves method, where location related variables are the most important to predict house prices, and makes us doubt the explanation provided by SHAP.

\begin{figure*}
    \centering
    \begin{subfigure}{.45\linewidth}
        \centering
        \includegraphics[width=\linewidth, height=\linewidth]{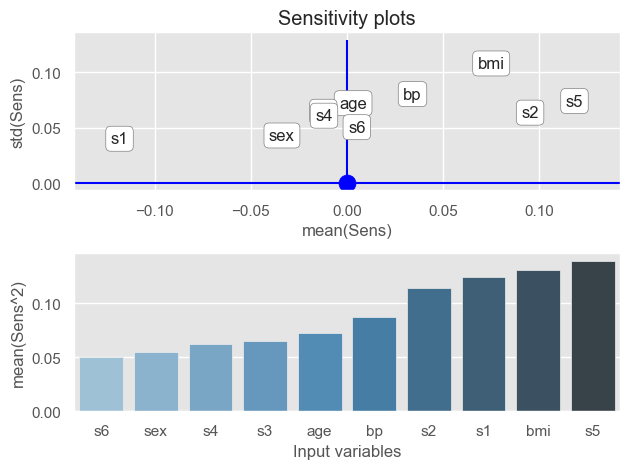}
        \caption{\label{fig:fig_diabetes_sens_plots} Sensitivity plots of diabetes dataset.}
        \end{subfigure}
    \begin{subfigure}{.45\linewidth}
        \centering
        \includegraphics[width=\linewidth, height=\linewidth]{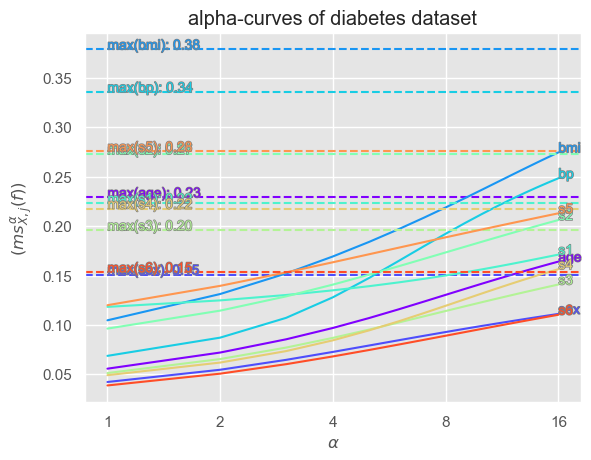}
        \caption{\label{fig:fig_diabetes_alpha_curves} $\alpha$-curves of diabetes dataset.}
    \end{subfigure}
    \begin{subfigure}{.45\linewidth}
        \centering
        \includegraphics[width=\linewidth, height=0.7\linewidth]{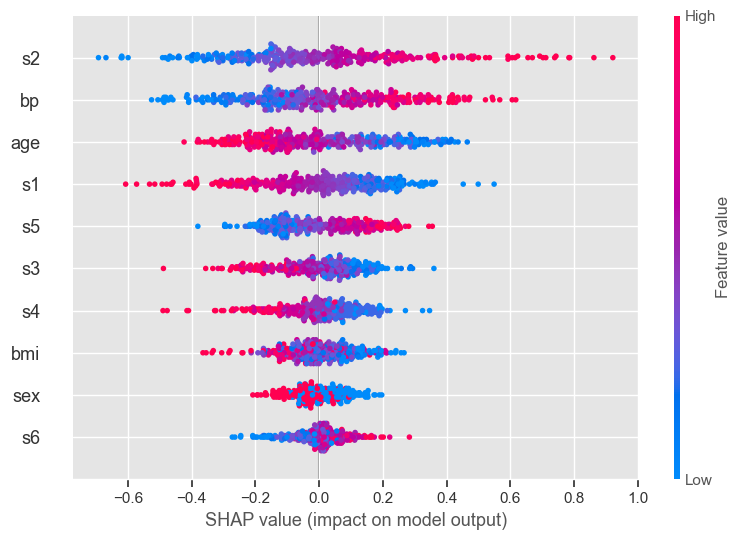}
        \caption{\label{fig:fig_diabetes_shap_values} SHAP values of diabetes dataset.}
    \end{subfigure}
    \begin{subfigure}{.45\linewidth}
        \centering
        \includegraphics[width=\linewidth, height=0.7\linewidth]{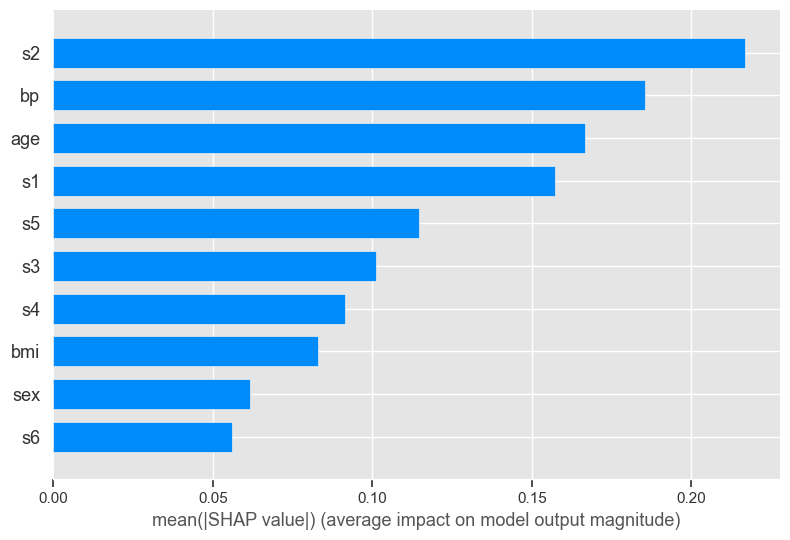}
        \caption{\label{fig:fig_diabetes_shap_importance} SHAP Importance values of diabetes dataset.}
    \end{subfigure}
    \begin{subfigure}{.45\linewidth}
        \centering
        \includegraphics[width=\linewidth, height=0.8\linewidth]{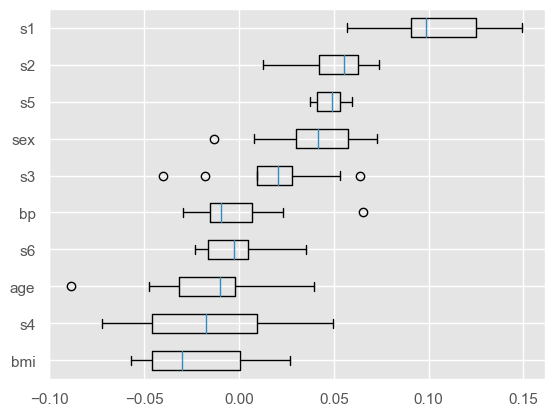}
        \caption{\label{fig:fig_diabetes_input_permutation} Input permutation importance values of diabetes dataset.}
    \end{subfigure}
    \caption{\label{fig:fig_diabetes_xai} XAI techniques plots of diabetes dataset. Techniques used are sensitivity analysis based on partial derivatives and SHAP.}
\end{figure*}

\subsubsection{Diabetes}
This dataset contains medical information of 442 diabetes patients with a quantitative measure of disease progression one year after baseline. It is composed of the following variables:
\begin{itemize}
    \item age: age in years
    \item sex
    \item bmi: body mass index
    \item bp: average blood pressure
    \item s1: tc, total serum cholesterol
    \item s2: ldl, low-density lipoproteins (LDL) cholesterol
    \item s3: hdl, high-density lipoproteins (HDL) cholesterol
    \item s4: tch, total cholesterol 
    \item s5: ltg, log of serum triglycerides level
    \item s6: glu, blood sugar level
\end{itemize}

Following the methodology for sensitivity analysis, figure \ref{fig:fig_diabetes_sens_plots} indicates that all variables have a non-linear relationship with the output, being the most important variables the serum triglycerides level ($s5$) and the body-mass index of the patient. According to literature \cite{tryg_diabetes_2021, tryg_diabetes_2022}, high triglycerides indicates an increase in cells insulin resistance, one of the causes of type 2 diabetes. Obesity (related to $bmi$) has been identified to be one of the main causes of diabetes \cite{Hjellvik_2012, gray_picone_sloan_yashkin_2015}. In \cite{Wang2021-ha} and \cite{Jiang2021-cg}, an index combination of both variables were an accurate predictor of high-risk groups of pre-diabetes. Regarding total serum cholesterol ($s1$), it seems counter-intuitive that a higher level of cholesterol indicates a fewer risk of diabetes (mean of sensitivities with respect to $s1$ is negative). Several studies have researched the cholesterol relationship to the probabilities of diabetes of a patient \cite{diabetes_1, diabetes_4, diabetes_6, diabetes_2, diabetes_3}. These studies determined that, although cholesterol has a relationship with the probabilities of a person developing diabetes, it is the LDL cholesterol level ($s2$) which might cause diabetes. It has also been determined that diabetes could cause high blood-pressure ($bp$) \cite{bpdiabetes_2022} and probabilities of suffering diabetes increase with the $age$ of the patient \cite{helmer_2022}. Respect to HDL and total cholesterol levels ($s3$ and $s4$), their low importance might be caused due to the information provided by these variables being sum up by $s1$ and $s2$ variables, and being used only to determine diabetes progression for patients with different pattern behavior for the variables $s1$ and $s2$. There can not be found notable differences in diabetes symptoms and causes on the patient based on gender ($sex$) in the consulted literature \cite{diabetes_7, diabetes_8}, and, although counter intuitive, blood sugar level was found the least important predictor in previous diabetes detection studies \cite{Hjellvik_2012, Wang2021-ha}.

Although results from sensitivity analysis are coherent with the literature reviewed, more information can be retrieved using the $\alpha-$curves methodology presented in this paper. Figure \ref{fig:fig_diabetes_alpha_curves} shows that $bmi$ is the most important variable for every alpha, and triglycerides level ($s5$) is one of the three most important variables for all alphas. However, there are regions of the dataset where the output is far more sensitive to the $bp$ variable, evolving from being the fifth most important to the second. A similar conclusion can be drawn for the LDL cholesterol level ($s2$), whose maximum sensitivity is similar to the maximum sensitivity of $s5$. The opposite happens to the total cholesterol ($s1$). Although mean sensitivity of $s1$ is notably high, its $\alpha-$curve barely rise, being the most linear output-input relationship for this dataset. Moreover, the probability of suffering diabetes is more sensitive to a change in the $age$ variable than in the $s1$ variable, as the maximum sensitivity of the former is greater than the maximum sensitivity of the latter. Based on this analysis, it might be concluded that there are some patients with remarkably different relationship between the output and input variables, where blood pressure and LDL cholesterol level play more important roles to predict the possibility of suffering diabetes than the triglycerides level, one of the main input variables used in the literature. Characterizing these patients could produce a more reliable model than the state-of-the-art.  

Analyzing SHAP results presented in figures \ref{fig:fig_diabetes_shap_values} and \ref{fig:fig_diabetes_shap_importance}, it can be observed that $s2$, $bp$, $age$ and $s1$ variables have a notably higher importance than the rest of the variables. The fact that not the $bmi$ nor the $s5$ variables are among the most important variables rises questions about the information retrieved based on the literature consulted. As SHAP assumes that the relationship between inputs and output is independent (and linear for each sample), interactions between these two variables are omitted and might be the cause of the low assigned importance to $bmi$ and $s5$. 

Figure \ref{fig:fig_diabetes_input_permutation} shows the input permutation results, which provide no conclusive order of importance along the variables. Depending on the permutation iteration analyzed, order of importance varies greatly and little information about the model can be retrieved. 

Based on the information retrieved from the previous methods (ignoring information from input permutation), the only common information is that, although counter-intuitive, the least important variable is the blood sugar level ($s6$), followed by the gender of the patient. Importance assigned for the rest of the variables depends on the method used to analyze the model, being sensitivity analysis and the $\alpha-$curves method the most coherent with the reviewed literature.

\section{Conclusions and future work}
\label{section:conclusions}
In this paper, we propose a novel XAI method to interpret ML models based on a metric interpretation of partial derivatives. Given a fitted ML model, the sensitivities of the output variable with respect to the inputs provide a relevant measure of the significance of the features in the problem analyzed. However, obtaining a meaningful metric to quantify feature importance based on the sensitivities remained an open issue.

This paper presents one major advantage with respect to previous works in this field. It provides theoretical proof and practical use of the $\alpha-$curves methodology, a novel technique to obtain feature importance information by aggregating the effect of sensitivities across the whole input space. 

 In this paper, the $\alpha-$curves methodology is applied to the analysis of Neural Networks. Nevertheless, it can be applied to any model whose partial derivatives can be calculated, allowing for a higher flexibility when facing a ML problem. 

The comparative analysis performed against other commonly used XAI methods (SHAP and input permutation) in synthetic and real datasets sharply demonstrate the effectiveness of the method to detect relevant features in the datasets. On the one hand, it provides information about the type of relationship between inputs and the output (linear, nonlinear). On the other hand, the method can detect if there exists regions in the input space where certain variables have a greater effect on the output than others. Detecting these regions is crucial, for example, to avoid removing features considered irrelevant when they only have a strong impact on a reduced area of the input space.

Finally, it is worth remarking that the proposed theoretical framework and the $\alpha$-curves methodology can be applied to both numerical and categorical features, provided that the data model embeds those variables in some $\mathbb{R}^n$. For instance, these methods can be used for studying neural networks trained over a dataset containing discrete variables, as they analyse the continuous model (the neural network, which receives real inputs) and not the dataset itself, whose categorical variables are embedded in the real inputs of the network. Nonetheless, using this theoretical framework as a starting point, more intrinsic methodologies for explaining categorical inputs or outputs could be examined in detail in future works by replacing the $L^p$ metrics by other suitable information metrics.

\section*{Annex I: Proof of the main theorem}
\label{annex:proofs}
This section contains the full mathematical proof of the main theorem \ref{thm:mainTheorem}.

Let us start the analysis with some useful notations. Given $\bar{h}=(h_1,\ldots,h_N)\in \mathbb{R}^N$ and $X=\{\bar{x}_i\}_{i=1}^n$, denote
$$\Delta_{\bar{h}} \bar{x}_i = (0,\ldots, \stackrel{j}{h_i},\ldots,0)\in \mathbb{R}^n$$
so that for each $i$
$$(x_{i,1},\ldots,x_{i,j}+h_{i},\ldots, x_{i,n})=\bar{x}_i+\Delta_{\bar{h}} \bar{x}_i.$$
Recall that $d_jf$ denotes the differential of $f$ with respect to variable $X_j$, i.e., if $f=(f_1,\ldots,f_m)$, then
$$d_jf= \left( \frac{\partial f_1}{\partial X_j} dX_j, \ldots, \frac{\partial f_m}{\partial X_j} dX_j\right).$$
Finally, let $\SD_{\SX,j}f:\mathbb{R}^N \longrightarrow \mathbb{R}^{Nm}$ be the linear operator
$$\SD_{\SX,j}=\bigoplus_{i=1}^N d_jf(\bar{x}_i).$$
It is then the operator that for each $\bar{h}\in \mathbb{R}^N$ yields
$$\SD_{\SX,j}f(\bar{h}):=\left(d_jf(\bar{x}_1)h_1,\ldots, d_jf(\bar{x}_N)h_N\right).$$
We can state the following estimate. Recall that given a linear map between finite dimensional normed spaces $A:(E,\lVert-\rVert_E)\longrightarrow (F,\lVert-\rVert_F)$ we define the norm of $A$ with respect to the norms of $E$ and $F$ as
$$\lVert A \rVert_{\lVert-\rVert_E, \lVert-\rVert_F}=\max_{\bar{x}\ne 0} \frac{\lVert A\bar{x} \rVert_F}{\lVert \bar{x} \rVert_E}.$$

\begin{theorem}
\label{thm:mainEstimate}
Let $f$ be a $\SC^2$ function. Then
$$s_{\SX,j}(f)= \lVert \SD_{\SX,j}f \rVert_{\lVert - \rVert_H, \lVert - \rVert_Y}.$$
\end{theorem}

\begin{proof}
By Taylor's Theorem, at each $\bar{x}_i\in \SX$
$$f(\bar{x}_i+\Delta_{\bar{h}}\bar{x}_i)-f(\bar{x}_i)= d_jf(\bar{x}_i) h_i+ r_i(h_i),$$
where $r_i(h_i)=o(|h_i|)$. Thus, we have
\begin{multline*}
    v_{\SX,j}(f,\bar{h}) = \lVert (d_jf(\bar{x}_i) h_i)_{i=1}^N + (r(h_i))_{i=1}^N \rVert_Y \\ 
    =\lVert \SD_{\SX,j}f(\bar{h})+r(\bar{h}) \rVert_Y,
\end{multline*}
where $r(\bar{h})=(r_1(h_1),\ldots,r_N(h_N))$. By triangular inequality, we have
\begin{multline*}
\sup_{\lVert \bar{h} \rVert=\varepsilon} \lVert \SD_{\SX,j}f(\bar{h}) \rVert_Y -\sup_{\lVert \bar{h} \rVert=\varepsilon} \lVert r(\bar{h}) \rVert_Y \le \sup_{\lVert \bar{h} \rVert=\varepsilon} v_{\SX,j}(f,\bar{h})\\
\le \sup_{\lVert \bar{h} \rVert=\varepsilon} \lVert \SD_{\SX,j}f(\bar{h}) \rVert_Y + \sup_{\lVert \bar{h} \rVert=\varepsilon} \lVert r(\bar{h}) \rVert_Y
\end{multline*}
so
\begin{multline*}
\lim_{\varepsilon \to 0}\frac{\sup_{\lVert \bar{h} \rVert=\varepsilon} \lVert \SD_{\SX,j}f(\bar{h}) \rVert_Y -\sup_{\lVert \bar{h} \rVert=\varepsilon} \lVert r(\bar{h}) \rVert_Y}{\varepsilon} \le  s_{\SX,j}(f)\\
\le \lim_{\varepsilon\to 0} \frac{\sup_{\lVert \bar{h} \rVert=\varepsilon} \lVert \SD_{\SX,j}f(\bar{h}) \rVert_Y + \sup_{\lVert \bar{h} \rVert=\varepsilon} \lVert r(\bar{h}) \rVert_Y}{\varepsilon}.
\end{multline*}
Observe that, by linearity of $\SD_{\SX,j}f(\bar{h})$,
$$\sup_{\lVert \bar{h} \rVert=\varepsilon} \lVert \SD_{\SX,j}f(\bar{h}) \rVert_Y = \lVert \SD_{\SX,j}f\rVert_{\lVert - \rVert_H,\lVert - \rVert_Y} \varepsilon.$$
On the other hand, as all norms on $Y=\mathbb{R}^{mN}$ are equivalent, we have that $r(\bar{h})=o(\lVert \bar{h} \rVert_Y)$, so
$$\lim_{\varepsilon \to 0} \frac{\sup_{\lVert \bar{h} \rVert=\varepsilon} \lVert r(\bar{h}) \rVert_Y}{\varepsilon}=0$$
and, therefore, we obtain that
\begin{multline*}
\lim_{\varepsilon \to 0} \frac{\sup_{\lVert \bar{h} \rVert=\varepsilon} \lVert \SD_{\SX,j}f(\bar{h}) \rVert_Y -\sup_{\lVert \bar{h} \rVert=\varepsilon} \lVert r(\bar{h}) \rVert_Y }{\varepsilon}\\
 = \lim_{\varepsilon \to 0} \frac{\sup_{\lVert \bar{h} \rVert=\varepsilon} \lVert \SD_{\SX,j}f(\bar{h}) \rVert_Y +\sup_{\lVert \bar{h} \rVert=\varepsilon} \lVert r(\bar{h}) \rVert_Y }{\varepsilon} \\ =\lVert \SD_{\SX,j}f\rVert_{\lVert - \rVert_H,\lVert - \rVert_Y}
\end{multline*}
and the Theorem follows.
\end{proof}

Using this operator norm framework allows us to analyze explicitly the cases where the norms $\lVert - \rVert_H$ and $\lVert - \rVert_Y$ are $L^p$-norms and to prove the main theorem \ref{thm:mainTheorem}. We will split the proof depending on which norm corresponds to an $L^p$ norm with the highest value of $p$.

From this point on, we will assume that $\lVert - \rVert_H$ is an $L^p$-norm and that $\lVert - \rVert_Y$ is an $L^q$ norm.

\subsection*{Proof of the main theorem when $p=q$}
\begin{theorem}
\label{thm:mainTheoremp=q}
Let $f$ be a $\SC^2$ function. If $\lVert - \rVert_H$ and $\lVert - \rVert_Y$  are the $L^p$ norms, then
$$s_{\SX,j}(f)=\max_{i}\left \{ \left \lVert d_jf(\bar{x}_i) \right \rVert_p \right\}.$$
\end{theorem}

\begin{proof}
By Theorem \ref{thm:mainEstimate}, we have
$$s_{\SX,j}(f)=\lVert \SD_{\SX,j}f \rVert_{p,p}.$$
Thus, we have to compute
$$\max_{\lVert \bar{h} \rVert_p=1} \lVert \SD_{\SX,j}f \rVert_p = \left( \sum_{i=1}^N \sum_{k=1}^m \left | \frac{\partial f}{\partial X_j}(\bar{x}_i) \right |^p |h_i|^p \right)^{1/p}$$
under the constraint $\sum_i |h_i|^p=1$. Changing the variable $H_i=|h_i|^p$ and rising to power $p$ the optimized function yields the following linear optimization problem.
$$\max_{\begin{array}{c}\sum_i H_i=1\\H_i\ge 0 \, \forall i\end{array}} \sum_{i=1}^N \left( \sum_{k=1}^m \left | \frac{\partial f}{\partial X_j}(\bar{x}_i) \right |^p \right) H_i$$
which is attained by taking $H_i$ to be $1$ when its coefficient is the biggest possible and $0$ in any other case. Thus
$$\max_{\begin{array}{c}\sum_i H_i=1\\H_i\ge 0 \, \forall i\end{array}} \sum_{i=1}^N \left( \sum_{k=1}^m \left | \frac{\partial f}{\partial X_j}(\bar{x}_i) \right |^p \right) H_i=\max_{i} \sum_{k=1}^m \left | \frac{\partial f}{\partial X_j}(\bar{x}_i) \right |^p$$
and, therefore,
\begin{multline*}
    \max_{\lVert \bar{h} \rVert_p=1} \lVert \SD_{\SX,j}f \rVert_p=\left( \max_{i} \sum_{k=1}^m \left | \frac{\partial f}{\partial X_j}(\bar{x}_i) \right |^p \right)^{1/p} \\ = \max_{i} \left(\sum_{k=1}^m \left | \frac{\partial f}{\partial X_j}(\bar{x}_i) \right |^p \right)^{1/p} = \max_{i} \lVert d_jf(\bar{x}_i) \rVert_p.
\end{multline*}
\end{proof}

\subsection*{Proof of the main theorem when $p>q$}
\begin{theorem}
\label{thm:mainTheoremp>q}
Let $f$ be a $\SC^2$ function. If $\lVert - \rVert_H=\lVert - \rVert_p$ and $\lVert - \rVert_Y=\lVert - \rVert_q$ with $p>q$, then
$$s_{\SX,j}(f)=\left (\sum_{i=1}^N \left( \sum_{k=1}^m \left |\frac{\partial f_k}{\partial X_j}(\bar{x}_i)\right |^q \right)^{\frac{p}{p-q}} \right )^{\frac{p-q}{pq}}.$$
\end{theorem}

\begin{proof}
As before, by Theorem \ref{thm:mainEstimate}, we have
$$s_{\SX,j}(f)=\lVert \SD_{\SX,j}f \rVert_{p,q}.$$
We have to compute
$$\lVert \SD_{\SX,j}f \rVert_{p,q}=\left( \max_{\lVert \bar{h}\rVert_p=1} \sum_{i=1}^N \left( \sum_{k=1}^m \left | \frac{\partial f_k}{\partial X_j}(\bar{x}_i) \right |^q\right) |h_i|^q \right)^{1/q}.$$
Changing the variable $H_i=|h_i|^q$, setting $\bar{H}=(H_1,\ldots,H_N)$ and rising to power $q$ the optimized function, we need to find the maximum of
$$\sum_{i=1}^N \left( \sum_{k=1}^m \left | \frac{\partial f_k}{\partial X_j}(\bar{x}_i) \right |^q\right) H_i$$
under the constraint $\lVert \bar{H}\rVert_{p/q}=1$. Taking $p'=\frac{p}{p-q}$ and $q'=p/q$ (so that $\frac{1}{p'}+\frac{1}{q'}=1$) and applying Hölder's Inequality with norms $L^{p'}$ and $L^{q'}$ yields
\begin{multline*}
    \sum_{i=1}^N \left( \sum_{k=1}^m \left | \frac{\partial f_k}{\partial X_j}(\bar{x}_i) \right |^q\right) H_i \le \left \lVert \left( \sum_{k=1}^m \left | \frac{\partial f_k}{\partial X_j}(\bar{x}_i) \right |^q \right)_{i=1}^N \right \rVert_{p'} \!\!\!\! \lVert \bar{H} \rVert_{q'} \\= \left \lVert \left( \sum_{k=1}^m \left | \frac{\partial f_k}{\partial X_j}(\bar{x}_i) \right |^q \right)_{i=1}^N \right \rVert_{p'}
\end{multline*}
with equality when $h_i \propto \left ( \sum_{k=1}^m \left | \frac{\partial f_k}{\partial X_j}(\bar{x}_i) \right |^q\right)^{p'/q'}$. Finally, taking the $q$-th root of this quantity yields
\begin{multline*}
    \lVert \SD_{\SX,j}f \rVert_{p,q} = \left ( \left \lVert \left( \sum_{k=1}^m \left | \frac{\partial f_k}{\partial X_j}(\bar{x}_i) \right |^q \right)_{i=1}^N \right \rVert_{p'}\right)^{1/q} \\ = \left (\sum_{i=1}^N \left (\sum_{k=1}^m \left | \frac{\partial f_k}{\partial X_j}(\bar{x}_i) \right |^q \right)^{p'} \right )^{q/p'}
\end{multline*}
obtaining the desired result.
\end{proof}

\subsection*{Proof of the main theorem when $p<q$}

\begin{theorem}
\label{thm:mainTheoremp<q}
Let $f$ be a $\SC^2$ function. If $\lVert - \rVert_H=\lVert - \rVert_p$ and $\lVert - \rVert_Y=\lVert - \rVert_q$ with $p<q$, then
$$s_{\SX,j}(f)=\max_{i}\left \{ \left \lVert d_jf(\bar{x}_i) \right \rVert_q \right\}.$$
\end{theorem}

\begin{proof}
By Theorem \ref{thm:mainEstimate}, we have
$$s_{\SX,j}(f)=\lVert \SD_{\SX,j}f \rVert_{p,q}.$$
Thus, we have to compute
$$\lVert \SD_{\SX,j}f \rVert_{p,q}=\left( \max_{\lVert \bar{h}\rVert_p=1} \sum_{i=1}^N \left( \sum_{k=1}^m \left | \frac{\partial f_k}{\partial X_j}(\bar{x}_i) \right |^q\right) |h_i|^q \right)^{1/q}.$$
Changing the variable $H_i=|h_i|^p$, setting $\bar{H}=(H_1,\ldots,H_N)$ and rising to power $q$ the optimized function, we need to find the maximum of
$$\sum_{i=1}^N \left( \sum_{k=1}^m \left | \frac{\partial f_k}{\partial X_j}(\bar{x}_i) \right |^q\right) H_i^{q/p}$$
under the constraint $\sum_i H_i=1$, with $H_i\ge 0$ for all $i$. As $q/p>1$, the objective functional is a convex function on $H_i$ and, thus, it attains its maximum value at one of the vertices of the domain. It is then clear that the maximum is attained taking $H_i=1$ precisely where the coefficient $\sum_{k=1}^m \left | \frac{\partial f_k}{\partial X_j}(\bar{x}_i) \right |^q$ attains its maximum value and $H_i=0$ for the rest of the values, so
\begin{multline*}
\max_{\lVert \bar{h} \rVert_p=1} \lVert \SD_{\SX,j}f \rVert_q=\left( \max_{i} \sum_{k=1}^m \left | \frac{\partial f_k}{\partial X_j}(\bar{x}_i) \right |^q \right)^{1/q}\\
= \max_{i} \left(\sum_{k=1}^m \left | \frac{\partial f_k}{\partial X_j}(\bar{x}_i) \right |^q \right)^{1/q} = \max_{i} \lVert d_jf(\bar{x}_i) \rVert_q.
\end{multline*}
\end{proof}

\bibliographystyle{IEEEtran}
\bibliography{ref}

\begin{thebibliography}{10}
\providecommand{\url}[1]{#1}
\csname url@samestyle\endcsname
\providecommand{\newblock}{\relax}
\providecommand{\bibinfo}[2]{#2}
\providecommand{\BIBentrySTDinterwordspacing}{\spaceskip=0pt\relax}
\providecommand{\BIBentryALTinterwordstretchfactor}{4}
\providecommand{\BIBentryALTinterwordspacing}{\spaceskip=\fontdimen2\font plus
\BIBentryALTinterwordstretchfactor\fontdimen3\font minus
  \fontdimen4\font\relax}
\providecommand{\BIBforeignlanguage}[2]{{%
\expandafter\ifx\csname l@#1\endcsname\relax
\typeout{** WARNING: IEEEtran.bst: No hyphenation pattern has been}%
\typeout{** loaded for the language `#1'. Using the pattern for}%
\typeout{** the default language instead.}%
\else
\language=\csname l@#1\endcsname
\fi
#2}}
\providecommand{\BIBdecl}{\relax}
\BIBdecl

\bibitem{Kohli_2018}
P.~S. Kohli and S.~Arora, ``Application of machine learning in disease
  prediction,'' in \emph{2018 4th International Conference on Computing
  Communication and Automation (ICCCA)}, 2018, pp. 1--4.

\bibitem{Shamout_2021}
F.~Shamout, T.~Zhu, and D.~A. Clifton, ``Machine learning for clinical outcome
  prediction,'' \emph{IEEE Reviews in Biomedical Engineering}, vol.~14, pp.
  116--126, 2021.

\bibitem{Rashid_2021}
M.~Rashid, B.~S. Bari, Y.~Yusup, M.~A. Kamaruddin, and N.~Khan, ``A
  comprehensive review of crop yield prediction using machine learning
  approaches with special emphasis on palm oil yield prediction,'' \emph{IEEE
  Access}, vol.~9, pp. 63\,406--63\,439, 2021.

\bibitem{dean_2022}
\BIBentryALTinterwordspacing
J.~Dean, ``Google research: Themes from 2021 and beyond,'' Jan 2022. [Online].
  Available:
  \url{https://ai.googleblog.com/2022/01/google-research-themes-from-2021-and.html}
\BIBentrySTDinterwordspacing

\bibitem{Samek_2021}
W.~Samek, G.~Montavon, S.~Lapuschkin, C.~J. Anders, and K.-R. Muller,
  ``Explaining deep neural networks and beyond: A review of methods and
  applications,'' \emph{Proc. IEEE Inst. Electr. Electron. Eng.}, vol. 109,
  no.~3, pp. 247--278, Mar. 2021.

\bibitem{Samek2019-mm}
W.~Samek, G.~Montavon, A.~Vedaldi, L.~K. Hansen, and K.-R. Muller, Eds.,
  \emph{Explainable {AI}: Interpreting, explaining and visualizing deep
  learning}, 1st~ed., ser. Lecture notes in computer science.\hskip 1em plus
  0.5em minus 0.4em\relax Cham, Switzerland: Springer Nature, Aug. 2019.

\bibitem{Barredo_2020}
\BIBentryALTinterwordspacing
A.~{Barredo Arrieta}, N.~Díaz-Rodríguez, J.~{Del Ser}, A.~Bennetot, S.~Tabik,
  A.~Barbado, S.~Garcia, S.~Gil-Lopez, D.~Molina, R.~Benjamins, R.~Chatila, and
  F.~Herrera, ``Explainable artificial intelligence (xai): Concepts,
  taxonomies, opportunities and challenges toward responsible ai,''
  \emph{Information Fusion}, vol.~58, pp. 82--115, 2020. [Online]. Available:
  \url{https://www.sciencedirect.com/science/article/pii/S1566253519308103}
\BIBentrySTDinterwordspacing

\bibitem{Benjamins_2019}
\BIBentryALTinterwordspacing
R.~Benjamins, A.~Barbado, and D.~Sierra, ``Responsible ai by design in
  practice,'' 2019. [Online]. Available: \url{https://arxiv.org/abs/1909.12838}
\BIBentrySTDinterwordspacing

\bibitem{Preece_2018}
\BIBentryALTinterwordspacing
A.~Preece, D.~Harborne, D.~Braines, R.~Tomsett, and S.~Chakraborty,
  ``Stakeholders in explainable ai,'' 2018. [Online]. Available:
  \url{https://arxiv.org/abs/1810.00184}
\BIBentrySTDinterwordspacing

\bibitem{Carvalho_2019}
\BIBentryALTinterwordspacing
D.~V. Carvalho, E.~M. Pereira, and J.~S. Cardoso, ``Machine learning
  interpretability: A survey on methods and metrics,'' \emph{Electronics},
  vol.~8, no.~8, 2019. [Online]. Available:
  \url{https://www.mdpi.com/2079-9292/8/8/832}
\BIBentrySTDinterwordspacing

\bibitem{Cheng_2018}
C.-H. Cheng, F.~Diehl, G.~Hinz, Y.~Hamza, G.~Nuehrenberg, M.~Rickert, H.~Ruess,
  and M.~Truong-Le, ``Neural networks for safety-critical applications —
  challenges, experiments and perspectives,'' in \emph{2018 Design, Automation
  and Test in Europe Conference and Exhibition}, 2018, pp. 1005--1006.

\bibitem{input_permutation_2021}
D.~Pawade, A.~Dalvi, J.~Gopani, C.~Kachaliya, H.~Shah, and H.~Shah, ``Xai---an
  approach for understanding decisions made by neural network,'' in
  \emph{Recent Trends in Communication and Intelligent Systems}, A.~K.
  Singh~Pundir, A.~Yadav, and S.~Das, Eds.\hskip 1em plus 0.5em minus
  0.4em\relax Singapore: Springer Singapore, 2021, pp. 155--165.

\bibitem{Shapley_1953}
L.~S. Shapley, ``A value for n-person games,'' pp. 307--318, 1953.

\bibitem{perm_shap_2020}
J.~Tritscher, M.~Ring, D.~Schlr, L.~Hettinger, and A.~Hotho, ``Evaluation of
  post-hoc xai approaches through synthetic tabular data,'' in
  \emph{Foundations of Intelligent Systems}, D.~Helic, G.~Leitner,
  M.~Stettinger, A.~Felfernig, and Z.~W. Ra{\'{s}}, Eds.\hskip 1em plus 0.5em
  minus 0.4em\relax Cham: Springer International Publishing, 2020, pp.
  422--430.

\bibitem{Hariharan2022}
\BIBentryALTinterwordspacing
S.~Hariharan, R.~R. Rejimol~Robinson, R.~R. Prasad, C.~Thomas, and
  N.~Balakrishnan, ``Xai for intrusion detection system: comparing explanations
  based on global and local scope,'' \emph{Journal of Computer Virology and
  Hacking Techniques}, Jul 2022. [Online]. Available:
  \url{https://doi.org/10.1007/s11416-022-00441-2}
\BIBentrySTDinterwordspacing

\bibitem{dimopoulos_use_1995}
I.~Dimopoulos, P.~Bourret, and S.~Lek, ``Use of some sensitivity criteria for
  choosing networks with good generalization ability,'' \emph{Neural Processing
  Letters}, vol.~2, pp. 1--4, 1995.

\bibitem{Gevrey_2003}
\BIBentryALTinterwordspacing
M.~Gevrey, I.~Dimopoulos, and S.~Lek, ``Review and comparison of methods to
  study the contribution of variables in artificial neural network models,''
  \emph{Ecological Modelling}, vol. 160, no.~3, pp. 249--264, 2003, modelling
  the structure of acquatic communities: concepts, methods and problems.
  [Online]. Available:
  \url{https://www.sciencedirect.com/science/article/pii/S0304380002002570}
\BIBentrySTDinterwordspacing

\bibitem{Gevrey_2006}
\BIBentryALTinterwordspacing
------, ``Two-way interaction of input variables in the sensitivity analysis of
  neural network models,'' \emph{Ecological Modelling}, vol. 195, no.~1, pp.
  43--50, 2006, selected Papers from the Third Conference of the International
  Society for Ecological Informatics (ISEI), August 26--30, 2002,
  Grottaferrata, Rome, Italy. [Online]. Available:
  \url{https://www.sciencedirect.com/science/article/pii/S0304380005005752}
\BIBentrySTDinterwordspacing

\bibitem{Durand2010}
\BIBentryALTinterwordspacing
E.~Durand-Cartagena and J.~Jaramillo, ``Pointwise lipschitz functions on metric
  spaces,'' \emph{Journal of Mathematical Analysis and Applications}, vol. 363,
  no.~2, pp. 525--548, 2010. [Online]. Available:
  \url{https://www.sciencedirect.com/science/article/pii/S0022247X0900780X}
\BIBentrySTDinterwordspacing

\bibitem{Szegedy2014}
\BIBentryALTinterwordspacing
C.~Szegedy, W.~Zaremba, I.~Sutskever, J.~Bruna, D.~Erhan, I.~Goodfellow, and
  R.~Fergus, ``Intriguing properties of neural networks,'' in
  \emph{International Conference on Learning Representations}, 2014. [Online].
  Available: \url{http://arxiv.org/abs/1312.6199}
\BIBentrySTDinterwordspacing

\bibitem{Thomas2022}
\BIBentryALTinterwordspacing
T.~Fel, D.~Vigouroux, R.~Cad{\`e}ne, and T.~Serre, ``{How Good is your
  Explanation? Algorithmic Stability Measures to Assess the Quality of
  Explanations for Deep Neural Networks},'' in \emph{{2022 CVF Winter
  Conference on Applications of Computer Vision (WACV)}}, Hawaii, United
  States, Jan. 2022. [Online]. Available:
  \url{https://hal.science/hal-02930949}
\BIBentrySTDinterwordspacing

\bibitem{molnar_2022}
\BIBentryALTinterwordspacing
C.~Molnar, ``Interpretable machine learning,'' Jul 2022. [Online]. Available:
  \url{https://christophm.github.io/interpretable-ml-book/}
\BIBentrySTDinterwordspacing

\bibitem{Letzgus_2022}
S.~Letzgus, P.~Wagner, J.~Lederer, W.~Samek, K.-R. M{\"u}ller, and G.~Montavon,
  ``Toward explainable artificial intelligence for regression models: A
  methodological perspective,'' \emph{IEEE Signal Processing Magazine},
  vol.~39, no.~4, pp. 40--58, 2022.

\bibitem{Minh_2022}
\BIBentryALTinterwordspacing
D.~Minh, H.~X. Wang, Y.~F. Li, and T.~N. Nguyen, ``Explainable artificial
  intelligence: a comprehensive review,'' \emph{Artificial Intelligence
  Review}, vol.~55, no.~5, pp. 3503--3568, Jun 2022. [Online]. Available:
  \url{https://doi.org/10.1007/s10462-021-10088-y}
\BIBentrySTDinterwordspacing

\bibitem{Speith_2022}
\BIBentryALTinterwordspacing
T.~Speith, ``A review of taxonomies of explainable artificial intelligence
  (xai) methods,'' in \emph{2022 ACM Conference on Fairness, Accountability,
  and Transparency}, ser. FAccT '22.\hskip 1em plus 0.5em minus 0.4em\relax New
  York, NY, USA: Association for Computing Machinery, 2022, p. 2239–2250.
  [Online]. Available: \url{https://doi.org/10.1145/3531146.3534639}
\BIBentrySTDinterwordspacing

\bibitem{bykov2020i}
K.~Bykov, M.~M.~C. Höhne, K.-R. Müller, S.~Nakajima, and M.~Kloft, ``How much
  can i trust you? -- quantifying uncertainties in explaining neural
  networks,'' 2020.

\bibitem{kastner_2021}
\BIBentryALTinterwordspacing
C.~Kästner, ``Interpretability and explainability,'' Jul 2021. [Online].
  Available:
  \url{https://ckaestne.medium.com/interpretability-and-explainability-a80131467856}
\BIBentrySTDinterwordspacing

\bibitem{scardi_developing_1999}
M.~Scardi and L.~W. Harding, ``\BIBforeignlanguage{en}{Developing an empirical
  model of phytoplankton primary production: A neural network case study},''
  \emph{\BIBforeignlanguage{en}{Ecological Modelling}}, vol. 120, no. 2-3, pp.
  213--223, 1999.

\bibitem{dimopoulos_neural_1999}
I.~Dimopoulos, J.~Chronopoulos, A.~Chronopoulou-Sereli, and S.~Lek, ``Neural
  network models to study relationships between lead concentration in grasses
  and permanent urban descriptors in athens city (greece),'' \emph{Ecological
  Modelling}, vol. 120, no. 2-3, pp. 157--165, 1999.

\bibitem{munoz_1998}
A.~Muñoz and T.~Czernichow, ``Variable selection using feedforward and
  recurrent neural networks,'' \emph{Engineering Intelligent Systems for
  Electrical Engineering and Communications}, vol.~6, no.~2, pp. 91--102, 1998.

\bibitem{white_statistical_2001}
H.~White and J.~Racine, ``Statistical inference, the bootstrap, and
  neural-network modeling with application to foreign exchange rates,''
  \emph{IEEE Transactions on Neural Networks}, vol.~12, no.~4, pp. 657--673,
  2001.

\bibitem{strumbelj_kononenko_2010}
\BIBentryALTinterwordspacing
E.~Strumbelj and I.~Kononenko, Jan 2010. [Online]. Available:
  \url{https://www.jmlr.org/papers/volume11/strumbelj10a/strumbelj10a.pdf}
\BIBentrySTDinterwordspacing

\bibitem{Lundberg_2017}
\BIBentryALTinterwordspacing
S.~M. Lundberg and S.-I. Lee, ``A unified approach to interpreting model
  predictions,'' in \emph{Advances in Neural Information Processing Systems
  30}, I.~Guyon, U.~V. Luxburg, S.~Bengio, H.~Wallach, R.~Fergus,
  S.~Vishwanathan, and R.~Garnett, Eds.\hskip 1em plus 0.5em minus 0.4em\relax
  Curran Associates, Inc., 2017, pp. 4765--4774. [Online]. Available:
  \url{http://papers.nips.cc/paper/7062-a-unified-approach-to-interpreting-model-predictions.pdf}
\BIBentrySTDinterwordspacing

\bibitem{pizarroso_2022}
\BIBentryALTinterwordspacing
J.~Pizarroso, J.~Portela, and A.~Muñoz, ``Neuralsens: Sensitivity analysis of
  neural networks,'' \emph{Journal of Statistical Software}, vol. 102, no.~7,
  p. 1–36, 2022. [Online]. Available:
  \url{https://www.jstatsoft.org/index.php/jss/article/view/v102i07}
\BIBentrySTDinterwordspacing

\bibitem{White_2001}
H.~White and J.~Racine, ``Statistical inference, the bootstrap, and
  neural-network modeling with application to foreign exchange rates,''
  \emph{IEEE Transactions on Neural Networks}, vol.~12, no.~4, pp. 657--673,
  2001.

\bibitem{yeh_cheng_2010}
\BIBentryALTinterwordspacing
I.-C. Yeh and W.-L. Cheng, ``First and second order sensitivity analysis of
  mlp,'' Jan 2010. [Online]. Available:
  \url{https://www.sciencedirect.com/science/article/abs/pii/S0925231210000639}
\BIBentrySTDinterwordspacing

\bibitem{zeng_zhen_he_han_2017}
\BIBentryALTinterwordspacing
X.~Zeng, Z.~Zhen, J.~He, and L.~Han, ``A feature selection approach based on
  sensitivity of rbfnns,'' Nov 2017. [Online]. Available:
  \url{https://www.sciencedirect.com/science/article/abs/pii/S092523121731706X}
\BIBentrySTDinterwordspacing

\bibitem{california_2022}
\BIBentryALTinterwordspacing
I.~Learning~Lab, ``The california housing dataset,'' 2022. [Online]. Available:
  \url{https://inria.github.io/scikit-learn-mooc/python_scripts/datasets_california_housing.html}
\BIBentrySTDinterwordspacing

\bibitem{diabetes_sklearn}
F.~Pedregosa, G.~Varoquaux, A.~Gramfort, V.~Michel, B.~Thirion, O.~Grisel,
  M.~Blondel, P.~Prettenhofer, R.~Weiss, V.~Dubourg, J.~Vanderplas, A.~Passos,
  D.~Cournapeau, M.~Brucher, M.~Perrot, and E.~Duchesnay, ``Scikit-learn:
  Machine learning in {P}ython,'' \emph{Journal of Machine Learning Research},
  vol.~12, pp. 2825--2830, 2011.

\bibitem{Kelley_1997}
\BIBentryALTinterwordspacing
R.~{Kelley Pace} and R.~Barry, ``Sparse spatial autoregressions,''
  \emph{Statistics and Probability Letters}, vol.~33, no.~3, pp. 291--297,
  1997. [Online]. Available:
  \url{https://www.sciencedirect.com/science/article/pii/S016771529600140X}
\BIBentrySTDinterwordspacing

\bibitem{tryg_diabetes_2021}
\BIBentryALTinterwordspacing
S.~Felson, ``Tryglicerides and diabetes,'' \emph{WebMD Medical Reference},
  2021. [Online]. Available:
  \url{https://www.webmd.com/diabetes/high-triglycerides}
\BIBentrySTDinterwordspacing

\bibitem{tryg_diabetes_2022}
\BIBentryALTinterwordspacing
N.~H. L.~B. Institute, ``High blood triglycerides,'' \emph{National Heart,
  Lung, and Blood Institute.}, 2022. [Online]. Available:
  \url{https://www.nhlbi.nih.gov/health-topics/high-blood-triglycerides}
\BIBentrySTDinterwordspacing

\bibitem{Hjellvik_2012}
V.~Hjellvik, S.~Sakshaug, and H.~Str{\o}m, ``\BIBforeignlanguage{en}{Body mass
  index, triglycerides, glucose, and blood pressure as predictors of type 2
  diabetes in a middle-aged norwegian cohort of men and women},''
  \emph{\BIBforeignlanguage{en}{Clin. Epidemiol.}}, vol.~4, pp. 213--224, Aug.
  2012.

\bibitem{gray_picone_sloan_yashkin_2015}
\BIBentryALTinterwordspacing
N.~Gray, G.~Picone, F.~Sloan, and A.~Yashkin, ``Relation between bmi and
  diabetes mellitus and its complications among us older adults,'' Jan 2015.
  [Online]. Available:
  \url{https://www.ncbi.nlm.nih.gov/pmc/articles/PMC4457375/}
\BIBentrySTDinterwordspacing

\bibitem{Wang2021-ha}
X.~Wang, J.~Liu, Z.~Cheng, Y.~Zhong, X.~Chen, and W.~Song,
  ``\BIBforeignlanguage{en}{Triglyceride glucose-body mass index and the risk
  of diabetes: a general population-based cohort study},''
  \emph{\BIBforeignlanguage{en}{Lipids Health Dis.}}, vol.~20, no.~1, p.~99,
  Sep. 2021.

\bibitem{Jiang2021-cg}
C.~Jiang, R.~Yang, M.~Kuang, M.~Yu, M.~Zhong, and Y.~Zou,
  ``\BIBforeignlanguage{en}{Triglyceride glucose-body mass index in identifying
  high-risk groups of pre-diabetes},'' \emph{\BIBforeignlanguage{en}{Lipids
  Health Dis.}}, vol.~20, no.~1, p. 161, Nov. 2021.

\bibitem{diabetes_1}
\BIBentryALTinterwordspacing
H.~Gylling, J.~A. Tuominen, V.~A. Koivisto, and T.~A. Miettinen, ``{Cholesterol
  Metabolism in Type 1 Diabetes},'' \emph{Diabetes}, vol.~53, no.~9, pp.
  2217--2222, 09 2004. [Online]. Available:
  \url{https://doi.org/10.2337/diabetes.53.9.2217}
\BIBentrySTDinterwordspacing

\bibitem{diabetes_4}
Y.~Lee and W.~J. Siddiqui, ``Cholesterol levels,'' in \emph{{StatPearls}
  [Internet]}.\hskip 1em plus 0.5em minus 0.4em\relax StatPearls Publishing,
  Jul. 2022.

\bibitem{diabetes_6}
M.~Budoff, ``Triglycerides and triglyceride-rich lipoproteins in the causal
  pathway of cardiovascular disease,'' \emph{The American Journal of
  Cardiology}, vol. 118, no.~1, p. 138–145, 2016.

\bibitem{diabetes_2}
\BIBentryALTinterwordspacing
Y.~C. Klimentidis, A.~Arora, M.~Newell, J.~Zhou, J.~M. Ordovas, B.~J. Renquist,
  and A.~C. Wood, ``{Phenotypic and Genetic Characterization of Lower LDL
  Cholesterol and Increased Type 2 Diabetes Risk in the UK Biobank},''
  \emph{Diabetes}, vol.~69, no.~10, pp. 2194--2205, 06 2020. [Online].
  Available: \url{https://doi.org/10.2337/db19-1134}
\BIBentrySTDinterwordspacing

\bibitem{diabetes_3}
\BIBentryALTinterwordspacing
C.~Perego, L.~{Da Dalt}, A.~Pirillo, A.~Galli, A.~L. Catapano, and G.~D.
  Norata, ``Cholesterol metabolism, pancreatic $\beta$-cell function and
  diabetes,'' \emph{Biochimica et Biophysica Acta (BBA) - Molecular Basis of
  Disease}, vol. 1865, no.~9, pp. 2149--2156, 2019. [Online]. Available:
  \url{https://www.sciencedirect.com/science/article/pii/S0925443919301231}
\BIBentrySTDinterwordspacing

\bibitem{bpdiabetes_2022}
\BIBentryALTinterwordspacing
``Diabetes and blood pressure.'' [Online]. Available:
  \url{https://www.diabetes.org.uk/guide-to-diabetes/managing-your-diabetes/blood-pressure}
\BIBentrySTDinterwordspacing

\bibitem{helmer_2022}
\BIBentryALTinterwordspacing
J.~Helmer, ``How age relates to type 2 diabetes,'' Apr 2022. [Online].
  Available: \url{https://www.webmd.com/diabetes/diabetes-link-age}
\BIBentrySTDinterwordspacing

\bibitem{diabetes_7}
\BIBentryALTinterwordspacing
L.~Arnetz, N.~R. Ekberg, and M.~Alvarsson, ``Sex differences in type 2
  diabetes: focus on disease course and outcomes,'' \emph{Diabetes, Metabolic
  Syndrome and Obesity: Targets and Therapy}, vol.~7, pp. 409--420, 2014, pMID:
  25258546. [Online]. Available:
  \url{https://www.tandfonline.com/doi/abs/10.2147/DMSO.S51301}
\BIBentrySTDinterwordspacing

\bibitem{diabetes_8}
\BIBentryALTinterwordspacing
P.~M. Nilsson, H.~Theobald, G.~Journath, and T.~Fritz, ``Gender differences in
  risk factor control and treatment profile in diabetes: a study in 229 swedish
  primary health care centres,'' \emph{Scandinavian Journal of Primary Health
  Care}, vol.~22, no.~1, pp. 27--31, 2004, pMID: 15119517. [Online]. Available:
  \url{https://doi.org/10.1080/02813430310003264}
\BIBentrySTDinterwordspacing

\end{thebibliography}

\end{document}